\documentclass[letterpaper]{article} 
\usepackage{aaai24}  
\usepackage{times}  
\usepackage{helvet}  
\usepackage{courier}  
\usepackage[hyphens]{url}  
\usepackage{graphicx} 
\usepackage{enumitem}
\usepackage{natbib}  
\usepackage{caption} 
\usepackage{algorithm}
\usepackage{algorithmic}
\usepackage{tabto}
\usepackage{amsmath, amssymb, dsfont, bm, amsthm, mathtools}
\usepackage[table, dvipsnames]{xcolor}
\usepackage{newfloat}
\usepackage{listings}
\usepackage{booktabs, array}
\usepackage[english]{babel}
\usepackage[caption=false]{subfig}
\usepackage{multirow}
\usepackage{newfloat}
\usepackage{listings}

\newcolumntype{C}[1]{>{\centering\let\newline\\\arraybackslash\hspace{0pt}}m{#1}}

\newtheorem{theorem}{Theorem}
\newtheorem*{theorem*}{Theorem}
\newcommand\scalemath[2]{\scalebox{#1}{\mbox{\ensuremath{\displaystyle #2}}}}

\DeclareCaptionStyle{ruled}{labelfont=normalfont,labelsep=colon,strut=off} 
\floatstyle{ruled}
\newfloat{listing}{tb}{lst}{}
\floatname{listing}{Listing}

\newcommand\Image[3][]{%
  \tabular[b]{@{}c@{}}\includegraphics[#1]{#2}\\
    #3
  \endtabular}

\title{Shaping Up SHAP: Enhancing Stability through Layer-Wise Neighbor Selection}
\author{
    Gwladys Kelodjou\textsuperscript{\rm 1},
    Laurence Roz\'e\textsuperscript{\rm 2},
    V\'eronique Masson\textsuperscript{\rm 1},
    Luis Gal\'arraga\textsuperscript{\rm 1}, \\
    Romaric Gaudel\textsuperscript{\rm 1},
    Maurice Tchuente\textsuperscript{\rm 3},
    Alexandre Termier\textsuperscript{\rm 1}
}

\affiliations {
    \textsuperscript{\rm 1}Univ Rennes, Inria, CNRS, IRISA - UMR 6074, F35000 Rennes, France\\
    \textsuperscript{\rm 2}Univ Rennes, INSA Rennes, CNRS, Inria, IRISA - UMR 6074, F35000 Rennes, France\\
    \textsuperscript{\rm 3}Sorbonne University, IRD,  University of Yaound\'e I, UMI 209 UMMISCO, P.O. Box 337 Yaound\'e, Cameroon\\
    gwladys.kelodjou@irisa.fr, laurence.roze@irisa.fr
}

\begin{document}
	\maketitle
	\begin{abstract}
        Machine learning techniques, such as deep learning and ensemble methods, are widely used in various domains due to their ability to handle complex real-world tasks. However, their black-box nature has raised multiple concerns about the fairness, trustworthiness, and transparency of computer-assisted decision-making. This has led to the emergence of local post-hoc explainability methods, which offer explanations for individual decisions made by black-box algorithms. Among these methods, Kernel SHAP is widely used due to its model-agnostic nature and its well-founded theoretical framework. Despite these strengths, Kernel SHAP suffers from high instability: different executions of the method with the same inputs can lead to significantly different explanations, which diminishes the relevance of the explanations.
        The contribution of this paper is two-fold. On the one hand, we show that Kernel SHAP's instability is caused by its stochastic neighbor selection procedure, which we adapt to achieve full stability without compromising explanation fidelity. On the other hand, we show that by restricting the neighbors generation to perturbations of size 1 -- which we call the coalitions of Layer 1 -- we obtain a novel feature-attribution method that is fully stable, computationally efficient, and still meaningful.
	\end{abstract}
	
	\section{Introduction and Related Work}

Modern machine learning techniques, such as deep learning and ensemble methods, have become invaluable tools in diverse high-stakes fields such as healthcare and banking, due to their accuracy at handling complex real-world tasks. Yet, their ``black-box" nature has raised concerns about their fairness and trustworthiness, which in turns has motivated the interest in explainability. In particular, {\em local post-hoc} explainability is concerned with providing a posteriori explanations for the decisions of black-box algorithms on single instances. Notable approaches in this line are LIME \cite{ribeiro2016should}, SHAP \cite{lundberg2017unified}, and Grad-CAM \cite{selvaraju2017grad}, which compute feature-attribution explanations. A feature attribution explanation assigns a contribution score to each of the model's input features. This tells users how (positively or negatively) and how much an input feature influenced the black-box's answer on a target instance.
    
Among the explainability methods based on feature attribution, SHAP is particularly popular because its attribution scores, the \textbf{SHAP values}, are inspired by the well-founded \textbf{Shapley values}~\cite{shapley1953value} from coalitional game theory.
In practice, SHAP implementations typically compute approximations of the theoretical SHAP values because of the exponential complexity of the exact calculation. For instance, the widely used \textbf{Kernel SHAP} -- SHAP's model-agnostic flavor -- estimates the attribution scores by solving a linear regression on a random sample of {\em neighbor instances}, constructed by perturbing the target instance. 

While local post-hoc explanations based on perturbations are widely used to explain black-box models (LIME also falls in this category), this class of methods are known to suffer from instability.~\citet{zhou2021s} argue that a key property of any explanation technique is its {\em stability} or reproducibility, defined as the method's ability to deliver the same results under the same conditions across multiple executions. They therefore propose a framework to determine the number of perturbation points required to ensure stable LIME explanations.~\citet{visani2022statistical} introduced two metrics to quantify the stability of LIME explanations: the Variable Stability Index, which measures the variation of the reported explanatory features, and the Coefficient Stability Index, which captures the variation of the coefficients associated to those features.~\citet{visani2020optilime} offer a framework that allows users to choose the level of stability and fidelity for LIME explanations.

We highlight the difference between {\em stability} and the criterion called {\em robustness} introduced by \citet{alvarez2018robustness}. An explanation method is robust if similar, but not identical, inputs entail similar explanations. \citet{alvarez2018robustness} quantify robustness as the maximum change in explanations obtained when small perturbations are applied to a given instance. They demonstrate that state-of-the-art methods (including LIME and SHAP) do not satisfy the robustness criterion. Stability and robustness are different concepts, however stability is a necessary condition for robustness.

While stability has been well studied for LIME, and stable improvements of LIME have been proposed, this is not the case for SHAP. In this paper, we show that the Kernel SHAP's estimator also suffers from stability issues, which stem from its stochastic neighbor selection. Given the popularity of Kernel SHAP, such instability diminishes the utility of SHAP, not to mention its potential impact on the trust in post-hoc explainability approaches. We list the contributions of this paper in the following.
    \begin{itemize}[leftmargin=*]
        \item Our first contribution is an alternative procedure for choosing the neighbors of a target instance in Kernel SHAP so that stability is improved.
        \item   Our second contribution is a thorough experimental study of the stability but also the fidelity of Kernel SHAP under different neighbor selection procedures. It shows that our proposed neighbor selection approach achieves stability without hurting explanation fidelity.
        \item Motivated by the results of our experiments, our third contribution is a novel approximation of the SHAP values, which relies on neighbors generated from simple perturbations -- called the \emph{layer-1 neighbors}. We demonstrate, via a theoretical study, that the obtained scores meet the main properties of attribution methods based on coalitional game theory. This means that our approach is an attribution method in itself, which provides scores that are stable and faster to compute than the SHAP values, while retaining their good properties. 
    \end{itemize}
    
    The remainder of this paper is organized as follows: in the \textit{Background} section, we provide an overview of SHAP and Kernel SHAP. The \textit{Reason for Instability} section brings to light the instability issues of Kernel SHAP. In the \textit{Improving Kernel SHAP’s Stability} section, we present our approach for improving the stability of Kernel SHAP. The \textit{Experiments} section provides an empirical evaluation of our proposed method. This experiment shows that using only {\em layer-1} neighbors already yields good results. Based on these results, in the \textit{First Layer Attribution Analysis} section, we conduct a theoretical analysis on the properties of our attribution method based on {\em layer-1} neighbors. We evaluate our attribution method in the latter part of this section by comparing the attributions of the exact SHAP values with those of {\em layer-1}. Finally, we conclude the paper in the last section, summarizing our findings and suggesting future research directions.

	\section{Background}
	\label{sec:background}

We now review the definitions of Shapley and SHAP values.
\subsection{Shapley Values}
 The Shapley value is a mathematical concept from game theory that was introduced by \citet{shapley1953value} to allocate performance credit across coalition game players. It is used to evaluate the \textit{fair share}  of a  player in a  cooperative game \cite{ghorbani2019data}. Suppose that we have a cooperative game with $M$ players numbered from $1$ to $M$, and let $N$ denote the set of players: $N = \{1, 2, \cdots, M\}$. A coalitional game is a tuple $\langle N, v \rangle$ where $v: 2^N \rightarrow \mathbb{R}$ is a characteristic function such that $v(\emptyset) = 0$ \cite{strumbelj2010efficient}. Subsets of $N$ are \textbf{coalitions} and $v$ returns a value for each coalition $S \,(S \subseteq N)$. So, for each coalition $S$, $v(S)$ represents the reward that arises when the players in the subset $S$ participate in the game. $v(S)$ can be seen as the worth of coalition $S$. The set $N$ is also a coalition, and we call it the \textbf{grand coalition}. Thus $v(N)$ is the total payoff that we aim to split among the players in a \textit{fair way}. Fair allocations must be based on each player’s contribution to the profit.
 The marginal contribution of a player $j$ to a coalition $S$ is defined as the additional value induced by including $j$ in the coalition: $v(S \cup \{j\}) - v(S)$.
 As we can see, this marginal contribution depends on the composition of the coalition. To determine the player's overall contribution, it is necessary to consider all possible subsets of players, i.e., coalitions.
Then, the Shapley value $\phi_j(v)$ of player $j$ for a game $v$ is calculated as a weighted average of the marginal contributions to the coalitions in which that player participates \cite{ruiz1998family}:
 \begin{equation}
		\scalemath{.95}{
  \phi_j(v) = \sum_{S \subseteq N\setminus\{j\}}\frac{|S|!(|N|-|S|-1)!}{|N|!} \left[v(S \cup \{j\}) - v(S)\right]. 
        }
            \label{eq:theorem_shapley}		
\end{equation} 

\noindent \citet{shapley1953value} proved that this value is the unique allocation of the grand coalition which satisfies the following axioms: Efficiency, Symmetry, Dummy, and Additivity. 

In recent years, the concept of Shapley value has become a popular technique for interpreting machine learning models~\cite{song2016shapley, ghorbani2019data, lundberg2017unified}. In this setting, the players are the features used by the model and the gain is the model's  prediction.  The aim is therefore to assign each feature an importance value that represents its effect on the model's prediction. Let $f$ be the predictive model to be explained, and $N = \{1, 2, \cdots, M\}$ be a set representing $M$ input features. $f$ maps an input feature vector $x = [x_1, x_2, \cdots, x_M]$ to an output $f(x)$ where $f(x) \in \mathbb{R}$ in regression and $f(x) \in [0,1]^{|C|}$ in classification, with $C$ being a finite set of labels. Shapley's equation thus becomes:

 \begin{equation}
 \scalemath{0.82}{
		\phi_j(f,x) = \sum_{S \subseteq N\setminus\{j\}}\frac{|S|!(|N|-|S|-1)!}{|N|!} \left[f_{S \cup \{j\}}(x_{S\cup \{j\}}) - f_S(x_S)\right],   
  }
\end{equation}
where $x_S$ represents the sub-vector of $x$ restricted to the feature subset $S$, and $f_S(x_S)$ is the model output restricted to the feature subset $S$.
To compute the Shapley values in such a setting, we must find out how to calculate $f_S(x_S)$, because most models cannot handle arbitrary patterns of missing input values -- they are only defined for the grand coalition. There are several approaches to do so. Here we will look at how SHAP~\cite{lundberg2017unified} works.

 \subsection{SHAP Values}

 SHAP values \cite{lundberg2017unified} are defined as the coefficients of an additive surrogate explanation model which is a linear function of binary features: 
 \begin{equation}
     g(z') = \phi_0 + \sum_{i=1}^M \phi_i z'_i.
     \label{eq:model_g}
 \end{equation}
  Here, $z' \in \{0,1\}^M$ is a binary vector representing a coalition of features, where $z'_i = 1$ when the $i$-th feature is observed and $z'_i = 0$ otherwise. The values $\phi_i \in \mathbb{R}$ correspond to the feature attribution values. 

  For example, a coalition $z' = [1, 0, 0, 1]$ simply represents the coalition $S = \{x_1, x_4\}$ if we assume $M = 4$. Equation \eqref{eq:shap_phis} gives the Shapley's values, where $|z'|$ is the number of $1$s in vector $z'$, and $z^\prime \setminus i$ is derived from $z'$ by setting $z'_i = 0$.
\begin{equation}
\scalemath{0.82}{
		\phi_i(f,x) = \sum_{z^\prime \in \{0,1\}^M}\frac{(|z'|-1)!(M-|z^\prime|)!}{M!}\left[f_x(z^\prime) - f_x(z^\prime \setminus i)\right]
  }
		\label{eq:shap_phis}
\end{equation} 

In Equation \eqref{eq:shap_phis}, $f_x(z')=\mathbb{E}[f(z) \,|\, z_S=x_S]$ with $S=\{i \in N \,|\, z'_i=1$\} and $z_S$ the sub-vector of $x$ restricted to the feature subset $S$. In order words $f_x(z')$ is the average of $f$'s answers when we fix the present features in $z'$, and consider all possible values of the absent features. 
SHAP values satisfy three desirable properties: \textit{local accuracy, missingness} and \textit{consistency} \cite{lundberg2017unified}, which are equivalent to those of the Shapley values (regarding local accuracy and consistency). 
 The exact computation of the SHAP values is exponential in the number of features, so numerous approaches have been proposed to approximate them. We are interested here in \textbf{Kernel SHAP}, which is the main model-agnostic variant of SHAP.
 \subsubsection{Kernel SHAP.} 
 It is a linear-regression-based approximation of the SHAP values. The idea is to find a surrogate function $g$ by solving a weighted least squares problem with weighting kernel $\pi_x$: 
 \begin{equation}
     \textit{arg min}_{g} \sum_{z'\in \{0,1\}^M} \pi_x(z') (g(z') - f_x(z'))^2
     \label{eq:wlr}
 \end{equation}
 \begin{equation}
			\pi_{x}(z') = \frac{M-1}{\binom{M}{|z'|} |z'|(M-|z'|)}
                \label{eq:pi}
\end{equation}
where $g$ in Equation \eqref{eq:wlr} is defined as in Equation \eqref{eq:model_g}. We note that the values $\pi_{x}(z') = \infty$ when $|z'| \in \{0, M\}$. This enforces constraints (i) $\phi_0 = f_x(\emptyset) = E[f(x)]$ for the intercept of $g$ and (ii) $\sum_{i = 1}^{M} \phi_i = f(x) - f_x(\emptyset)$ for the sum of the coefficients. The latter constraint is the \textit{local accuracy} constraint, which states that the sum of the attribution coefficients equals the total payoff, the prediction $f(x)$ in this case. 
To approximate the SHAP values, Kernel SHAP samples a subset of coalitions and then uses the sample to solve Equation~\eqref{eq:wlr}.
The size of the sample used to learn Kernel SHAP is called the \textbf{budget}. Given a budget, the sampling is guided by the coalition weights ($\pi_{x}$, also called \textit{Shapley Kernel}). The weight assigned to each coalition is determined by the number of features involved in that coalition (denoted as $|z'|$ and representing the number of $1s$ in $z'$). From now we use the terms neighbor and coalition interchangeably.

\textbf{Layer} $i$ is defined as the set of coalitions having $i$ features present or $i$ features absent (or $M-i$ features present). All the coalitions within the same layer have the same weight. Table \ref{tab:coalition_and_layer} is an example of some coalitions and their corresponding layers for $M=4$. As $i$ approaches $\lfloor\frac{M}{2}\rfloor$, the weight assigned to the coalitions within that layer decreases. Consequently, the sampling procedure progresses incrementally, starting from lower-level layers and gradually moving towards higher-level layers. Specifically, this involves generating first the coalitions with $1$ or $M-1$ features included. If the budget is not exhausted, then coalitions with $2$ or $M-2$ features included are generated, and so on. 
\begin{table}[!htb]
    \centering
		
		\begin{tabular}{ccc}
			\toprule
                 Layer && Coalition \\ \midrule
			$1$ && $1\quad0\quad0\quad0$ \\ 
			$1$ && $0\quad1\quad1\quad1$  \\
			$2$ && $1\quad1\quad0\quad0$ \\
			$2$ && $0\quad0\quad1\quad1$ \\\bottomrule
			 
		\end{tabular}          
		\caption{Examples of coalitions and their corresponding layers for $M=4$. The first coalition is in Layer 1, as it contains only one feature. The second coalition is still in Layer 1 as it lacks only one feature. In the last two cases, two features are present or absent, hence the instances belong to Layer 2.}
		\label{tab:coalition_and_layer}
	\end{table}
Based on the allocated budget, a layer may be either fully enumerated or not. A layer is considered \textbf{complete} if all possible coalitions within that layer have been enumerated. Conversely, if there are coalitions from a layer that have not been generated, we will refer to it as an \textbf{incomplete} layer.
When focusing on a specific layer, there are two options: either fully explore that layer and then move on to the next layer, or switch to random sampling mode on all the other unexplored layers. For a layer to be fully explored, two conditions must be met: 
\begin{enumerate}
    \item The remaining budget must be sufficient to generate all the instances of that layer.
    \item The weight of the layer (sum of the weights of all instances within the layer) must be significant enough to justify its exploration. Specifically, the weighted probability of drawing an instance of that layer must be greater than the probability of randomly selecting any instance.
\end{enumerate} 
The third column in Table \ref{tab:sampling_process} provides an example of SHAP's layer-wise exploration for a budget of $1200$ coalitions, and illustrates how Kernel SHAP switches to random selection after layer 2, despite having enough budget to fully explore layer 3. The fourth column of the table will be explained in the {\em Improving Kernel SHAP’s Stability} section.

\begin{table}
\centering

    \begin{tabular}{ccC{1.2cm}C{2.2cm}}
        \toprule
        \textbf{Layer} & \textbf{Size} & \textbf{Kernel SHAP}          & \textbf{ST-SHAP {\footnotesize (Our method)}}\\ \midrule
        $1$            & $30$          & $30$                   & $30$                                       \\ \midrule
        $2$            & $210$         & $210$                  & $210$                                      \\ \midrule
        $3$            & $910$         & \multirow{5}{*}{$960$} & $910$                                      \\ \cmidrule(r){1-2} \cmidrule(l){4-4} 
        $4$            & $2730$        &                        & $50$                                       \\ \cmidrule(r){1-2} \cmidrule(l){4-4} 
        $5$            & $6006$        &                        & \multirow{3}{*}{$0$}                       \\ \cmidrule(r){1-2}
        $6$            & $10010$       &                        &                                            \\ \cmidrule(r){1-2}
        $7$            & $12870$       &                        &                                            \\ \bottomrule
        \end{tabular}
    
    \caption{Kernel SHAP and ST-SHAP (our method) neighbor selection process for $M=15$ and budget = $1200$. The column ``Size'' is the number of coalitions in the layer, whereas the columns ``Kernel SHAP'' and ``ST-SHAP'' denote the number of coalitions from a layer materialized by the two different generation processes. }
\label{tab:sampling_process}
\end{table}

Finally, similarly to LIME, SHAP offers the option to control the complexity of the explanation. This can be achieved by specifying the number of features to include, i.e. the number of non-zero coefficients provided by the linear explainer. This step ensures the {\em interpretability} of the explanation \cite{ribeiro2016should,10.1145/3411764.3445315}. 

To summarize, Kernel SHAP is a method for approximating the SHAP values that relies on a weighted linear regression. The regression coefficients are estimates of the SHAP values. The weighted linear model is learned from  a sample of the coalitions. The sampling procedure prioritizes lower-level layers whenever possible, and else draws randomly coalitions from higher-level layers. For each coalition, the prediction is obtained using a black-box predictive function $f_x$, and the weight is computed using the Shapley Kernel (Equation \eqref{eq:pi}). Finally, the weighted linear model is fitted, according to Equation \eqref{eq:wlr}, using the weighted coalitions and their corresponding predictions.

	\section{Reason for Instability}
        \label{sec:reasons_instability}
    In this section, we illustrate the instability of Kernel SHAP with an example shown in Figure \ref{fig:shap_instability} and present an hypothesis on the reason of this instability. In this figure, we use an instance of the Dry Bean dataset\footnote{Dry Bean Dataset. \url{ https://doi.org/10.24432/C50S4B}. Accessed: 2024-02-01.} and compute three explanations on the same instance with identical parameters. We note that the four relevant features are not the same across these three executions.
    
    To quantify stability, we compute explanations for the same instance multiple times under the same configuration and compute the  Jaccard coefficient of the sets of features with non-zero coefficients in the explanations.  
    If $S_1, \dots, S_n$ denote those sets, this corresponds to:
    \begin{equation}
        J(S_1, S_2, \cdots, S_n) = \frac{|S_1 \cap S_2 \cap \cdots \cap S_n|}{|S_1 \cup S_2 \cup \cdots \cup S_n|}
        \label{eq:jaccard}
    \end{equation}
    The coefficient ranges between 0 and 1, and values closer to 1 denote better stability.
     
     In our previous example, the Jaccard coefficient obtained for $20$ explanations (on the same instance) is $0.74$ (with a budget set to $200$). Further results on the instability of Kernel SHAP are presented in the \textit{Experiments} section. 

 \begin{figure*}[!htb]
     \centering
    \subfloat[Input instance]{\includegraphics[scale=.2]{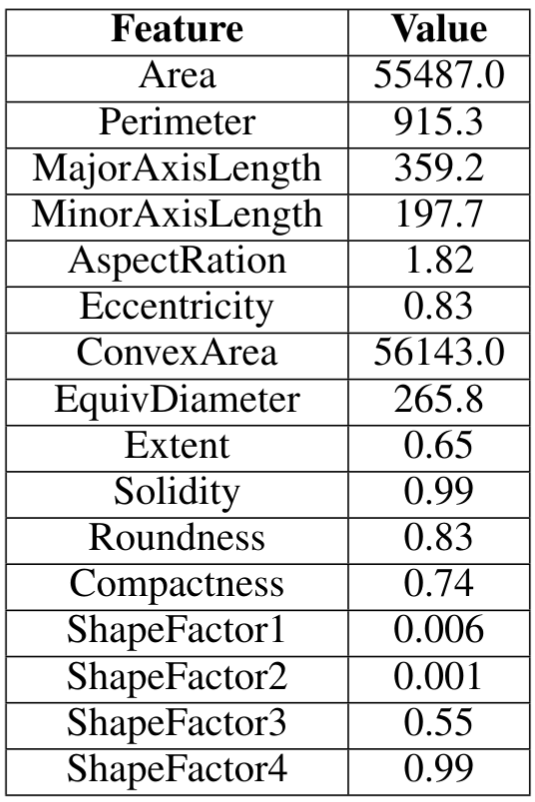}}
    \hspace{.7cm}
    \subfloat[Execution 1]{\includegraphics[scale=.35]{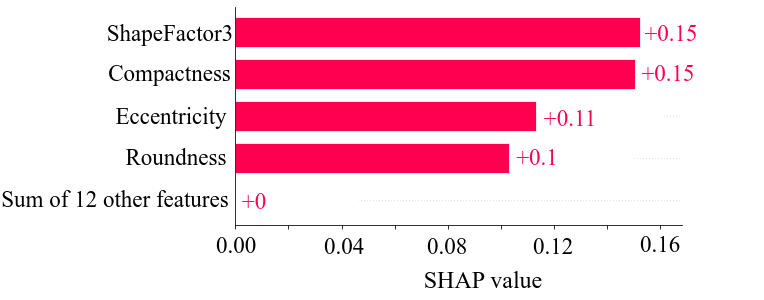}}
    
    \subfloat[Execution 2]{\includegraphics[scale=.35]{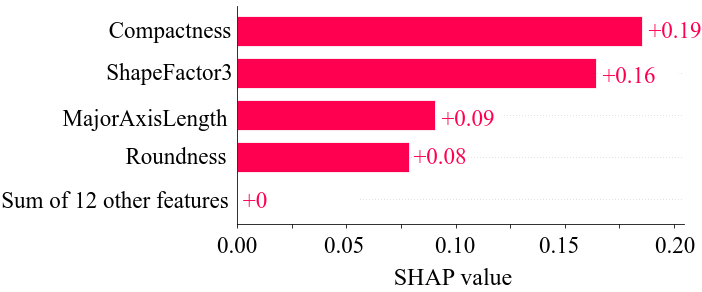}}\hfill
    \subfloat[Execution 3]{\includegraphics[scale=.35]{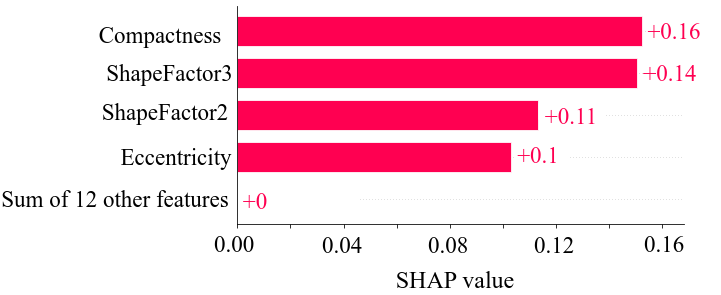}}
    
    \caption{Illustration of the instability of \textsc{SHAP} on an instance from the Dry Bean dataset. The explanation size is 4, so each explanation computation yields four features with non-zero values.}
\label{fig:shap_instability}
 \end{figure*}

    We believe that the main source of instability in Kernel SHAP explanations lies in the way coalitions are generated. 
    As seen in the previous section, Kernel SHAP's neighbor selection strategy introduces a lot of randomness in the process. 
    The generation of coalitions used to fit the linear model involves layer-wise filling whenever possible, followed by random sampling (see Table \ref{tab:sampling_process}). In practice only the lower-level layers are fully explored, whereas the higher layers are (sub-)sampled from a large and diverse population based on the remaining budget.
    Those samples can be very different across multiple executions of Kernel SHAP on the same instance. We believe that this stochastic process is the primary source of Kernel SHAP's instability. 
		
	\section{Improving Kernel SHAP’s Stability}
        \label{sec:contrib1}

	Our goal is to reduce the variability in explanations across multiple executions of Kernel SHAP on the same instance, and if possible, to achieve complete determinism. To do so, we propose to drop Kernel SHAP's condition $2$ when filling a layer (see Section \textit{Background}, subsection \textit{Kernel SHAP}, condition $2$ to fully explore a layer), that is, fill the layers in increasing order as long as there is still enough budget. Unlike Kernel SHAP, if the budget is not sufficient to fill a layer, random generation only takes place in that layer -- instead of all the remaining layers. Then, to eliminate variability completely, we can set the budget in a way that only complete layers are considered.
 
Let us consider the case where $M = 15$ depicted in Table~\ref{tab:sampling_process}, which illustrates Kernel SHAP's neighbor generation process for 7 layers. To generate elements of the first layer, we consider the cases where only one feature is included and the cases where only 1 feature is excluded. We can easily enumerate all possible coalitions within this layer, resulting in exactly $2 \cdot C_{15}^1 = 30$ coalitions.
If the allocated budget is $30$, we can be sure that the same 30 coalitions will be generated every time we run SHAP. In that line of thought, budgets of $30$ or $240$ ($30+210$) guarantee full determinism in SHAP because they entail complete layers.

Now, if the budget is $1200$ as in the example, Kernel SHAP will saturate layers 1 and 2 (with $30$ and $210$ neighbors respectively), which leaves us with a remaining budget of $1200 - 210 - 30 = 960$ instances. Kernel SHAP will sample those instances from all the remaining layers, whereas we propose to saturate layer $3$ with $910$ neighbors and sample the remaining $50$ neighbors randomly from layer $4$. As our experiments show, such an approach, that we call \textbf{ST-SHAP} (\textit{Stable SHAP}), mitigates the deleterious effects of sub-sampling in the stability of the explanations. This is so because we restrict the random sampling of neighbors to a smaller and less diverse universe, i.e., layer $4$ in ST-SHAP vs. layers $3$ to $7$ in Kernel SHAP.

We will also show the benefits of setting the number of neighbors in a way that allows only complete layers, thereby completely eliminating variability.
	\section{Experiments}
        \label{sec:contrib2}

    We conduct a series of experiments to verify the effectiveness of ST-SHAP at mitigating the instability of Kernel SHAP.
    \subsection{Experimental Protocol}
    \subsubsection{Environment.}
    The experiments are conducted on a computer with a 12th generation Intel-i7 CPU (14 cores with HT, 2.3-4.7 GHz Turbo Boost).  
    We use the official Kernel SHAP implementation provided by the authors~\cite{lundberg2017unified}. ST-SHAP\footnote{\url{https://github.com/gwladyskelodjou/st-shap}} is based on this implementation. We use the scikit-learn python library\footnote{\url{https://scikit-learn.org/}} to fit the black-box models that we aim to explain. 
    \subsubsection{Datasets.} 
    We use several datasets from the UCI Machine Learning Repository and other popular real-world datasets, namely: Boston \cite{harrison1978hedonic} and Movie,
    which are regression datasets, and Default of Credit Card Clients ~\cite{misc_default_of_credit_card_clients_350}, Adult~\cite{misc_adult_2}\footnote{We use a reformatted version released by the authors of SHAP.},  Dry Bean\footnote{\url{https://doi.org/10.24432/C50S4B}. Accessed: 2024-02-01.}, Spambase \cite{spambase}, HELOC\footnote{\url{https://community.fico.com/s/explainable-machine-learning-challenge?tabset-158d9=3}. Accessed: 2024-02-01.}, and Wisconsin Diagnostic Breast Cancer \cite{wdbcdataset}, which are classification datasets. Table \ref{tab:datasets} summarizes the information about the experimental datasets.

    \begin{table}[t]
    \fontsize{10}{12}\selectfont
    \resizebox{\linewidth}{!}{
        \begin{tabular}{ccccC{1.2cm}} 
            \toprule
            Dataset      & Size  & \#features & Task        & {\small Explanation size} \\ \midrule
            Movie        & 505   & 19        & Regression  & 4   \\ 
            Boston       & 506   & 13        & Regression    & 4  \\
            Dry Bean     & 13611 & 16        & Classification & 4 \\
            Credit Card       & 30000 & 23        & Classification & 4\\
            Adult   & 32561 & 12        & Classification & 4 \\
            Spambase & 4601 & 57 & Classification & 4 \\
            HELOC  & 10459 & 23 & Classification & 4 \\
            WDBC  &  569 & 30 &  Classification & 4 \\ \bottomrule
        \end{tabular}
        }
        \caption{Summary of datasets used in our experiments. The {\em Explanation size} indicates the number of non-zero coefficients returned by the linear model of the explanation. 
        These are the most significant features associated to a particular prediction (the absolute value of each coefficient indicates the magnitude of that importance). 
        }
	\label{tab:datasets}
    \end{table}

    \subsubsection{Black-Box Models.} 
    We also use various black-box models, including Support Vector Machine (SVM), Random Forest (RF), Logistic Regression (LR), and Multi-layer Perceptron (MLP). For each dataset, we randomly select 10 instances from the test set. Finally, we compute the explanation 20 times for each explained instance
    to derive the corresponding Jaccard value that estimates stability.
    \subsubsection{Metrics.} 
    We evaluate stability by measuring the evolution of the Jaccard coefficient (Equation~\eqref{eq:jaccard}) as we change the budget. We test budgets that lead to complete layers for layers $1$ to $3$, as well as budgets that do not: \{$10$, $20$, $30$, $40$, $50$, $60$, $70$, $80$, $90$, $100$, $200$, $500$, $1000$, $2000$\}. 

\subsection{Results}
    \begin{figure*}[!htb]

     \centering
    \subfloat[Boston]{\includegraphics[scale=0.15]{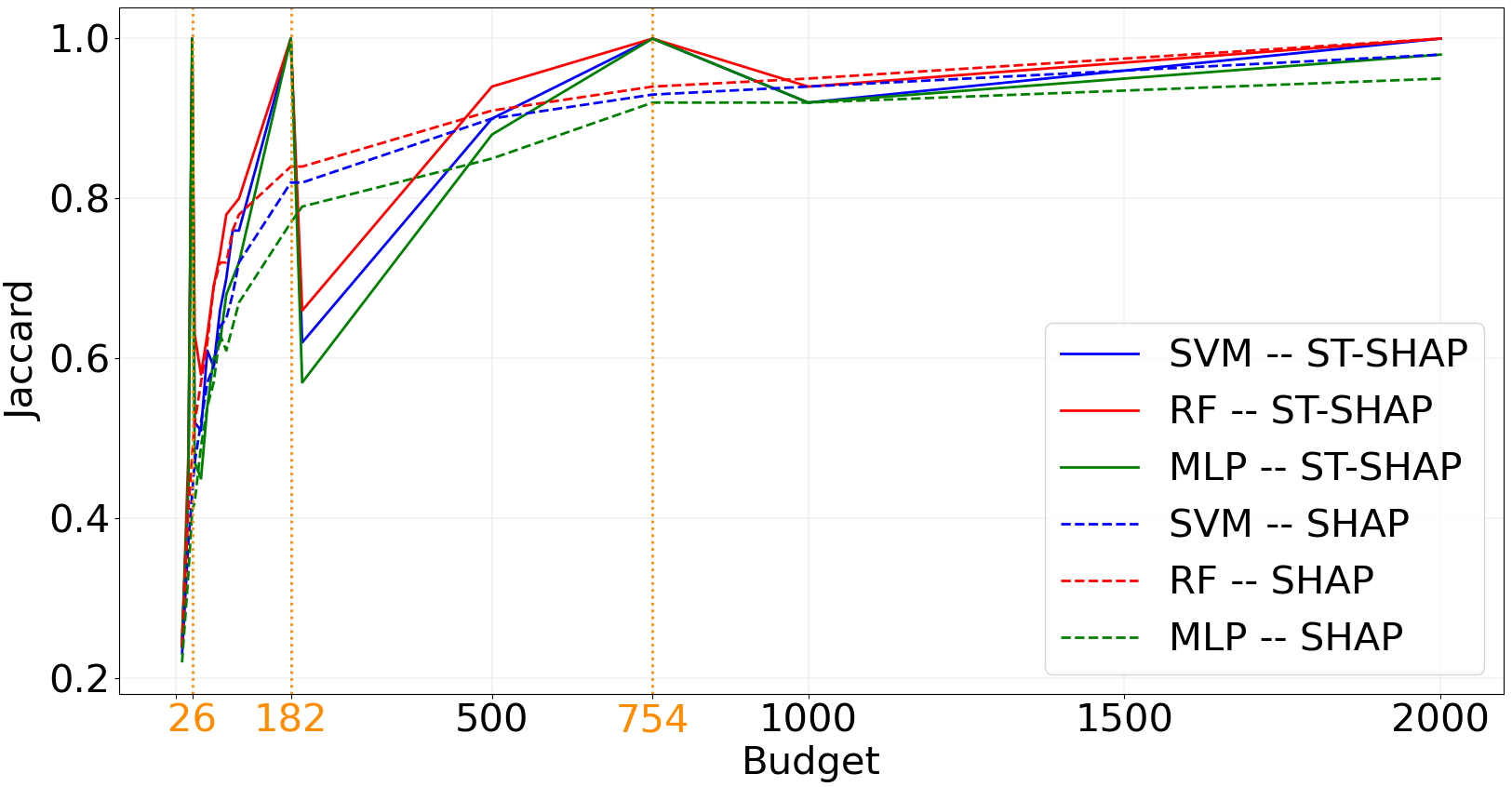}}
    \hspace{.1cm}
    \subfloat[Dry Bean]{\includegraphics[scale=0.15]{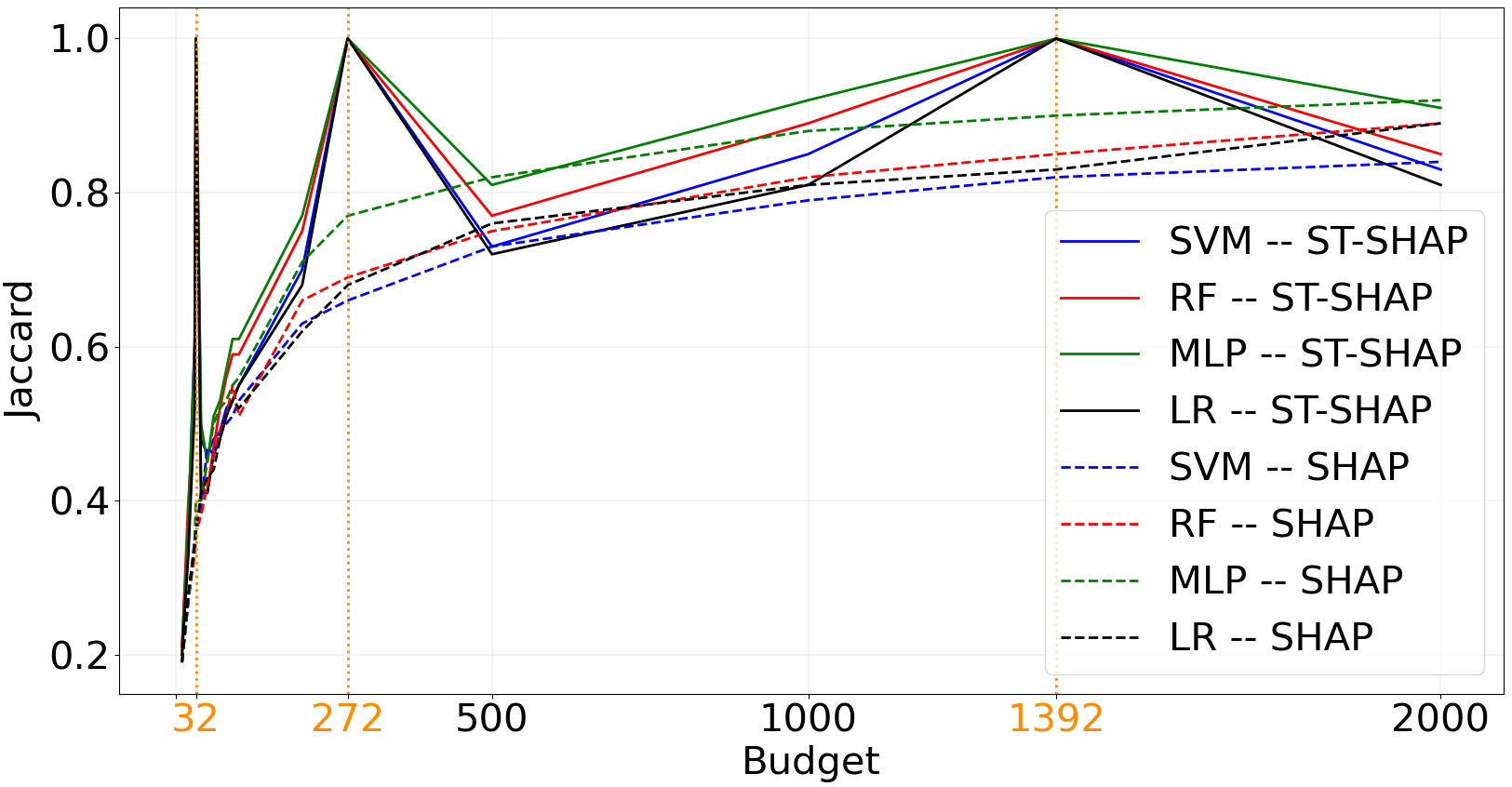}}
    
    \subfloat[HELOC]{\includegraphics[scale=0.15]{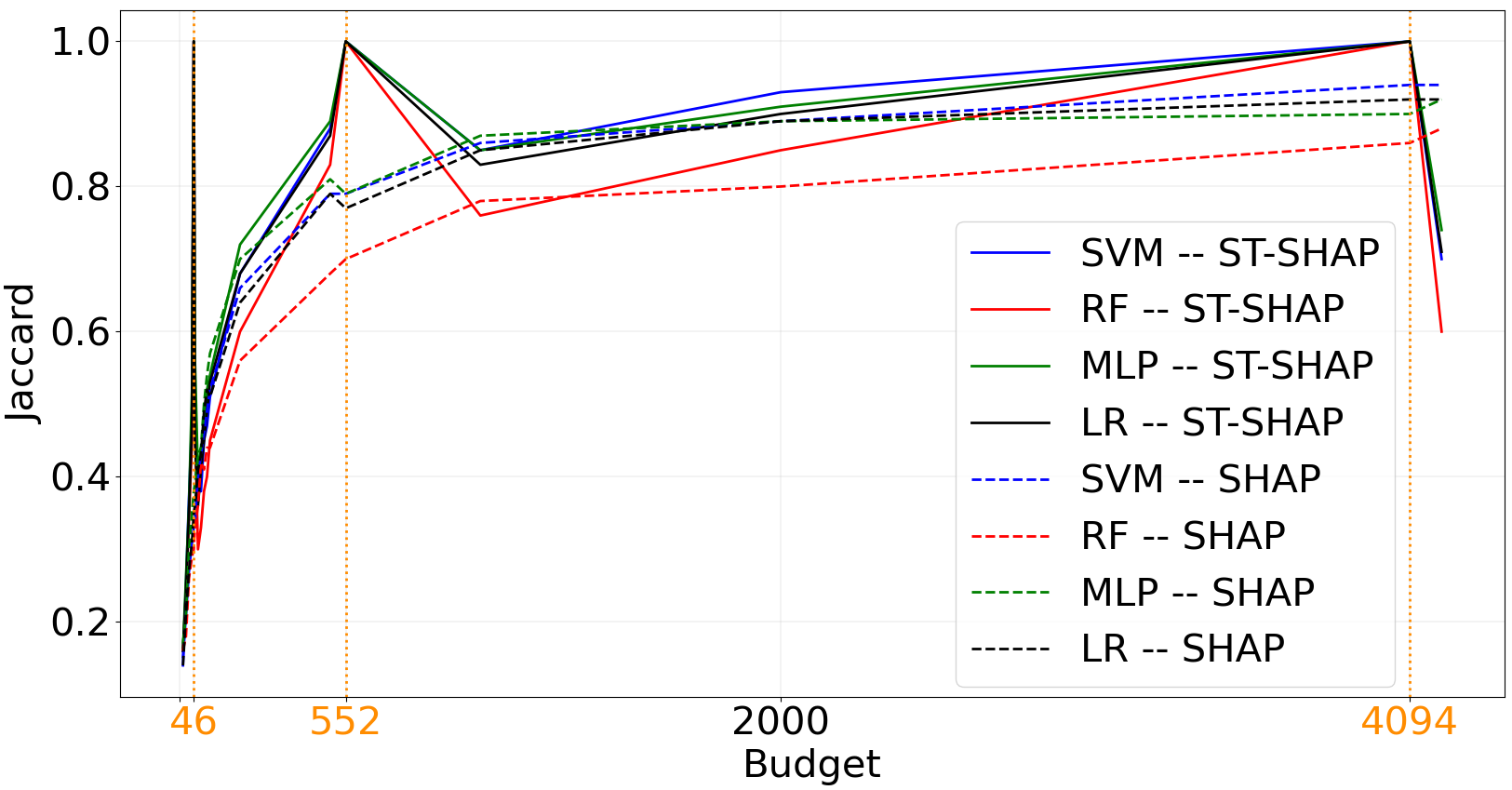}}
    \hspace{.1cm}
    \subfloat[Spambase]{\includegraphics[scale=0.15]{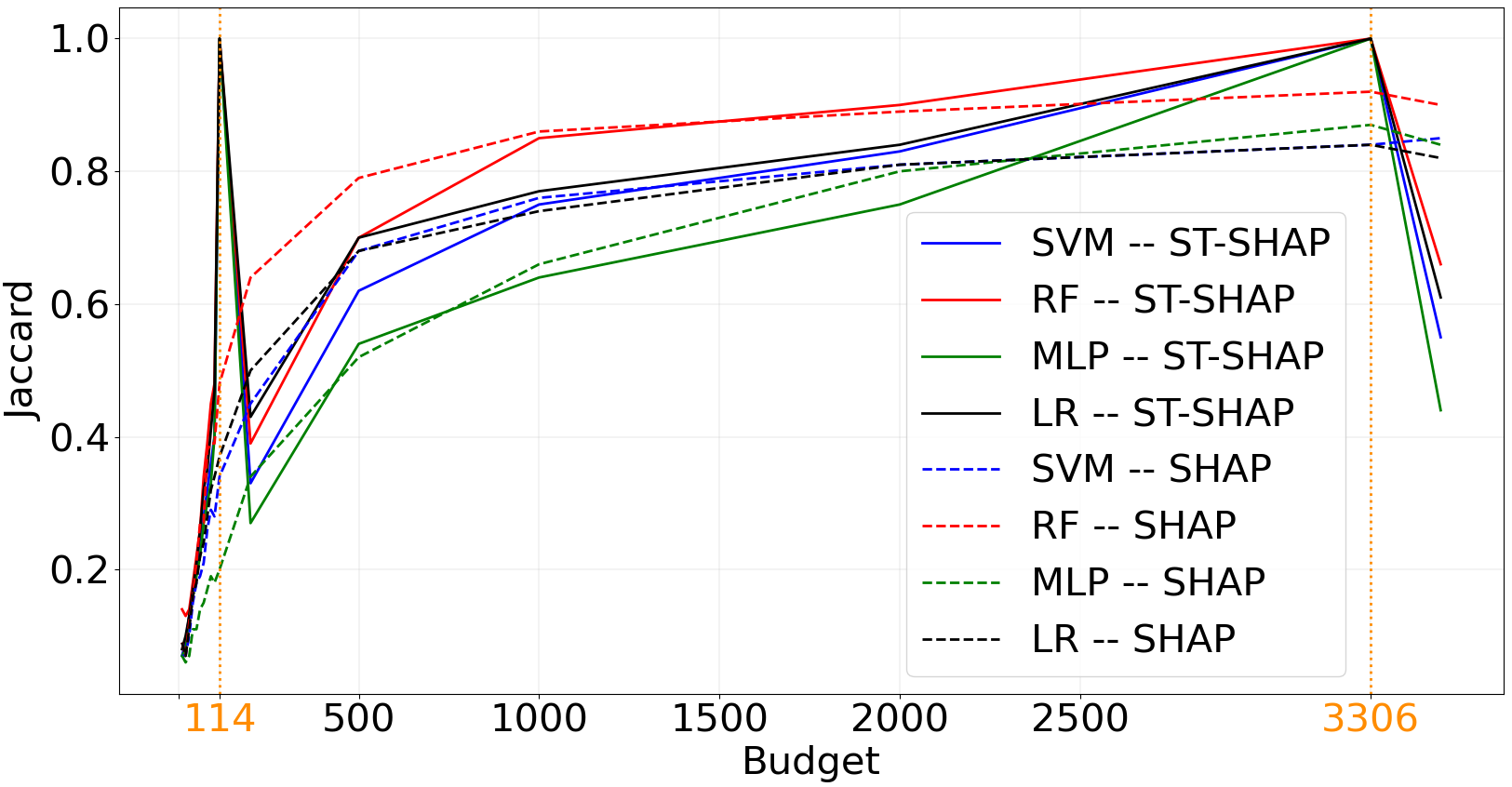}}
    
    \caption{Evolution of the Jaccard coefficient for SHAP and ST-SHAP on Boston, Dry Bean, HELOC, and Spambase datasets. The vertical lines represent the budgets that result in a complete layer.}
\label{fig:jaccard_datasets}
 \end{figure*}
For the sake of conciseness, we refer to Kernel SHAP as SHAP in all the charts. 
        Figure~\ref{fig:jaccard_datasets}a presents the results obtained on the Boston dataset. 
        We report the average Jaccard coefficient among the 10 instances for each of the studied budgets.
        Layers 1, 2 and 3 become complete when the budgets are $26$, $182$ and $754$ respectively. We can see that, for ST-SHAP, the Jaccard coefficient equals $1.0$ when the budget corresponds to complete layers. This indicates a complete absence of variability in the process to compute the explanation. We also notice that ST-SHAP's stability outperforms Kernel SHAP's as we approach budgets with complete layers.
        In addition to these observations, we highlight in Figure \ref{fig:jaccard_datasets}b with the Dry Bean dataset that ST-SHAP is generally more stable than Kernel SHAP, even on incomplete layers. Figure \ref{fig:jaccard_datasets}d showcases results for Spambase dataset, a larger dataset. For complete layers, only layers 1 and 2 have been calculated, as layer 3 requires a budget of $61826$ coalitions. Budgets for incomplete layers remain unchanged, except with the addition of a budget of $3500$ coalitions. We observe the same trends as in the previous datasets.
        The results presented here cover one regression and three classification datasets. The results for the other datasets can be found in \cite[Appendix~2]{kelodjou2024shapingupshap}. Similar observations apply to these datasets.
        
        Since stable explanations of low fidelity are as useless as unstable explanations, 
        we examine the impact of our method on the fidelity of the explanations. 
        To assess fidelity, we compare the outputs of the black-box model with those of the Kernel SHAP and ST-SHAP linear approximation models. For this purpose, we use the $R^2$ statistic for regression tasks (using the Boston and the Movie datasets) and the {\em accuracy} for classification tasks (for other datasets). In both cases the metrics assess the quality of the linear approximation provided by Kernel SHAP or ST-SHAP at ``adhering'' to the black box's outcomes -- so the higher they are the better. This form of fidelity is referred to as \emph{adherence} in the literature~\cite{visani2020optilime, gaudel2022s}.

    \begin{figure*}[!htb]

     \centering
    \subfloat[Boston]{\includegraphics[scale=.11]{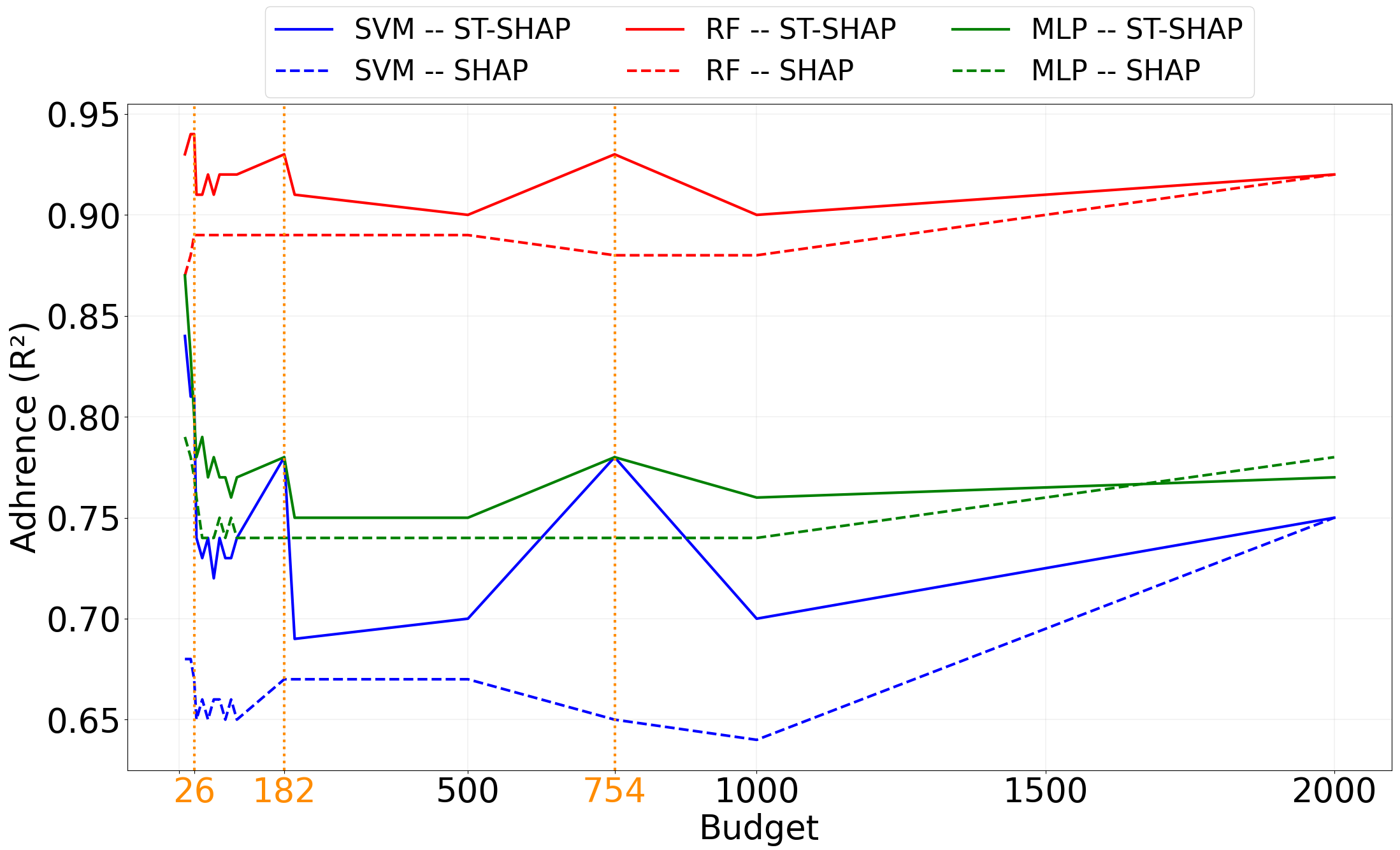}}
    \hspace{.08cm}
    \subfloat[Dry Bean]{\includegraphics[scale=.11]{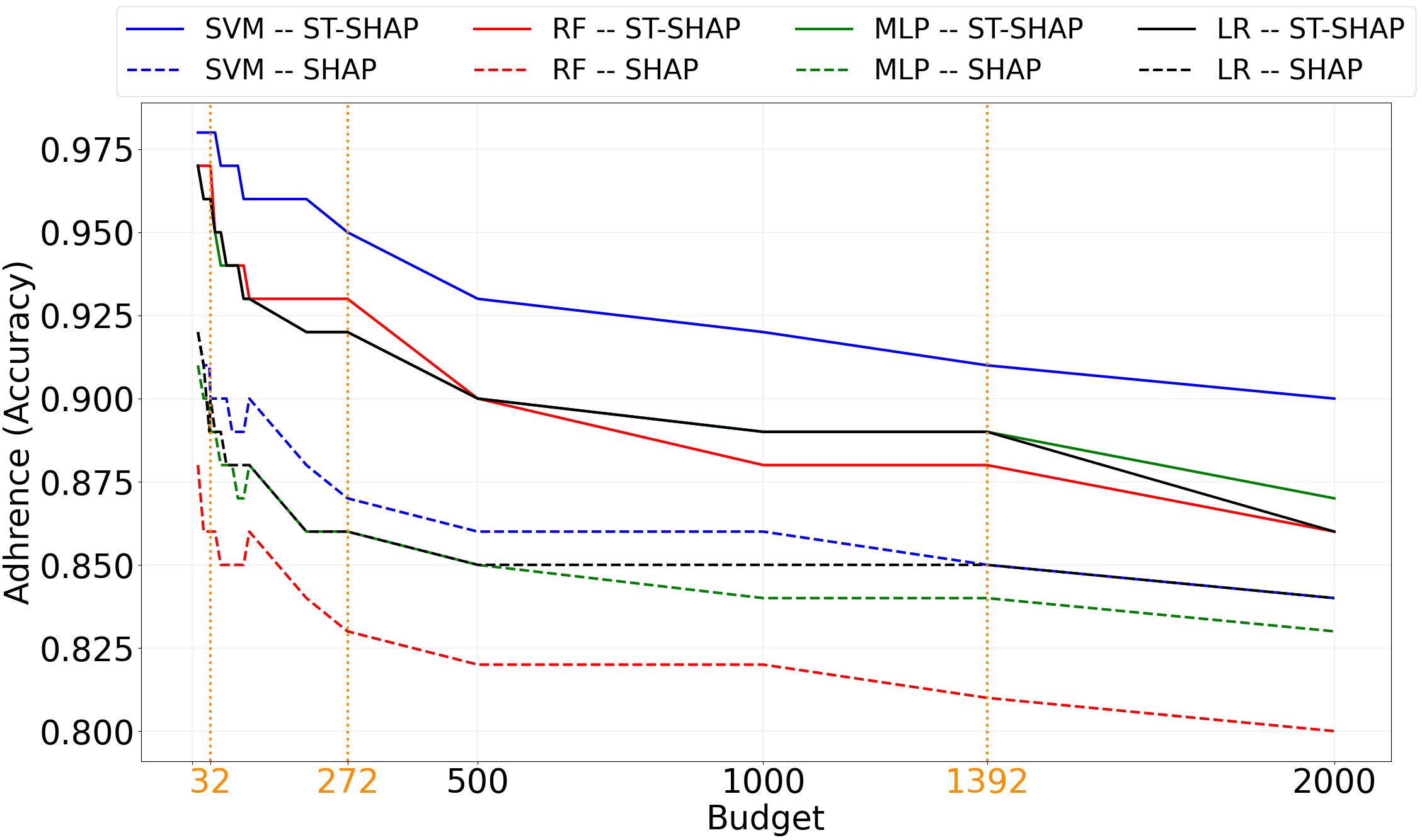}}
    
    \subfloat[HELOC]{\includegraphics[scale=.11]{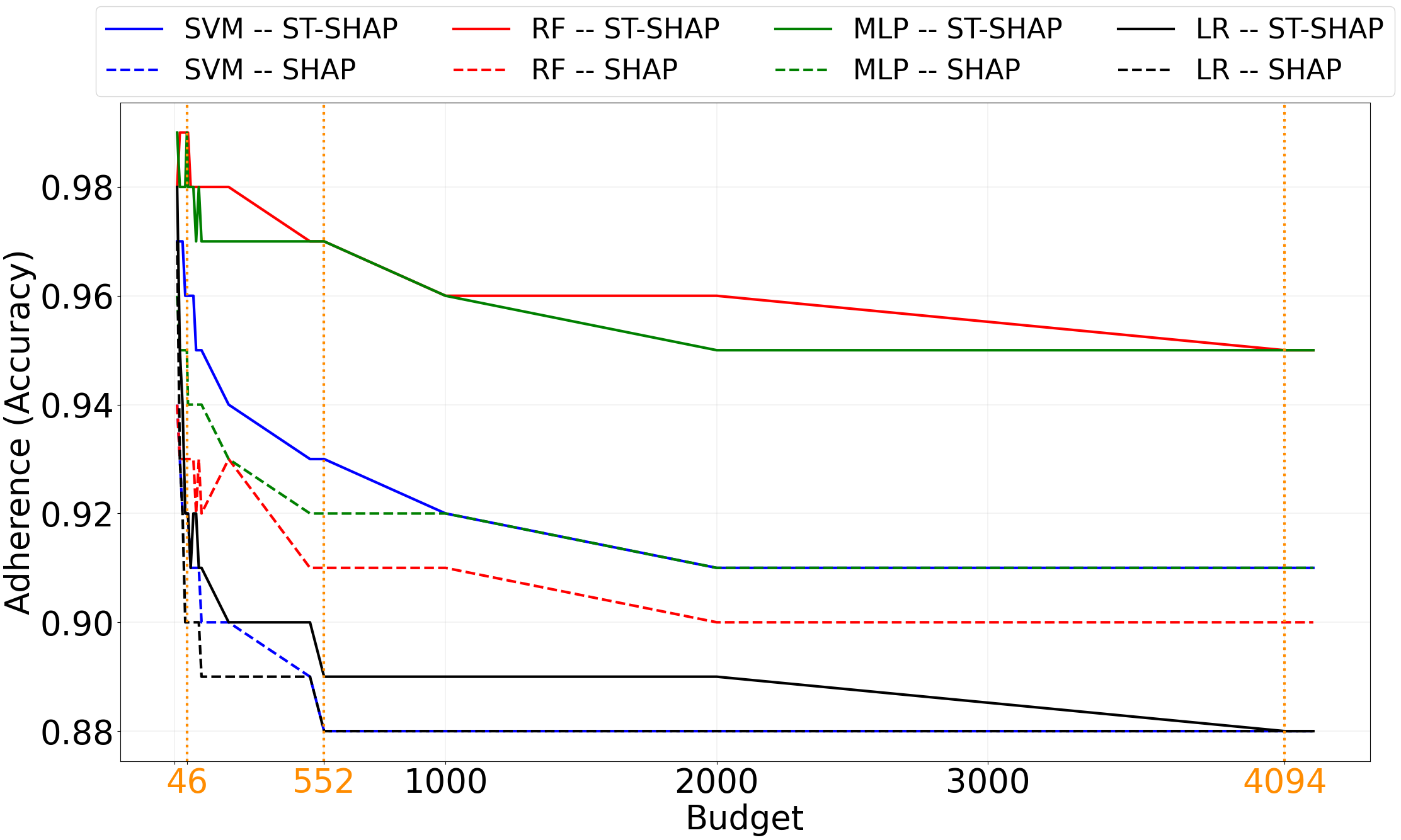}}
    \hspace{.08cm}
    \subfloat[Spambase]{\includegraphics[scale=.11]{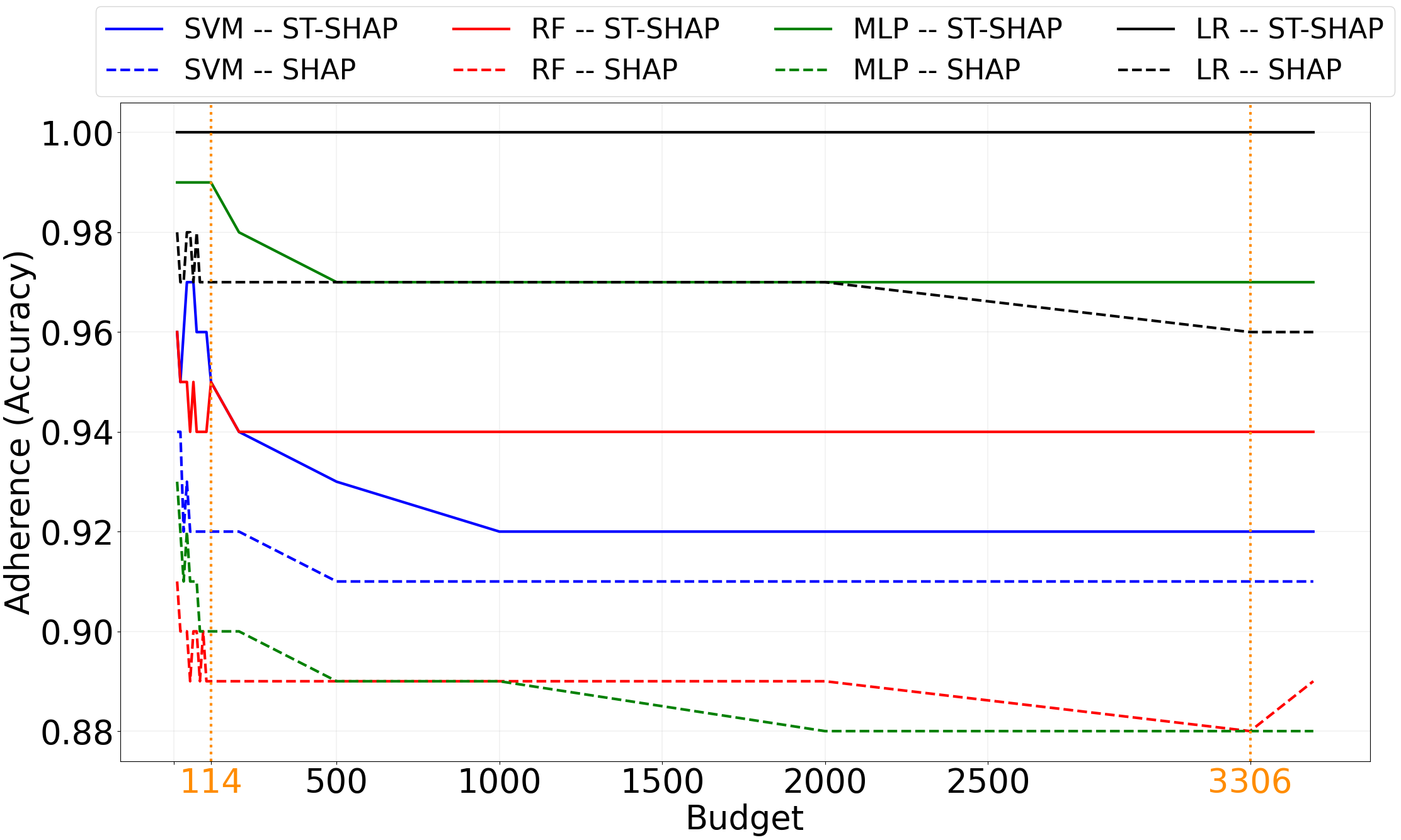}}
    
    \caption{Adherence ($R^2$-score for Boston and Accuracy for others) vs. budget size for Boston, Dry Bean, HELOC, and Spambase datasets.}
    \label{fig:adherence_datasets}
    \end{figure*}

        Figure \ref{fig:adherence_datasets}a is produced by taking the average $R^2$-score obtained for each budget.
        Each $R^2$-score is obtained by comparing the results of Kernel SHAP (resp. ST-SHAP) on the set of coalitions used to train the linear explainer. From this figure, it can be observed that the fidelity of ST-SHAP is comparable to Kernel SHAP's. It is also noteworthy that the fidelity is generally high, even for small budgets (e.g., by considering only the neighbors in Layer 1). Figure \ref{fig:adherence_datasets}b illustrates the fidelity comparison between Kernel SHAP and ST-SHAP on the Dry Bean dataset. The results demonstrate a generally good fidelity for both methods, with ST-SHAP sometimes outperforming Kernel SHAP. 
       Furthermore, for both methods, fidelity tends to decrease with an increasing budget. This trend can be attributed to the fact that, as the budget increases, we sample neighbors from upper layers, which confronts the linear regression with a more diverse set of coalitions. 
       Depending on the black-box, this diversity can make the optimization problem more challenging, resulting in lower fidelity compared to a regression trained on the first layers. In Figure \ref{fig:adherence_datasets}d, where we present the results for the Spambase dataset, it can be observed that ST-SHAP generally exhibits better fidelity than Kernel SHAP.
       
        The results for the remaining datasets are provided in \cite[Appendix~2]{kelodjou2024shapingupshap}. 
        These datasets yield the same conclusions. 
        Additional experiments comparing the exact SHAP values with the coefficients obtained with ST-SHAP show that the approximation of the SHAP values provided by ST-SHAP is as good as that of Kernel SHAP \cite[Appendix~2]{kelodjou2024shapingupshap}. 
        
        Our experiments on stability and fidelity lead us to conclude that our layer-wise strategy to complete layers of coalitions in Kernel SHAP reduces variability satisfactorily, leading to stability. This stability is maximal (Jaccard score equals 1.0) when the budget guarantees fully complete layers.  Moreover, our methods do not have an impact on the adherence of the explanations. This means that users can now choose which layers to use depending on the available time and computing resources, since a larger number of coalitions translates into longer execution times.       
        After examining the stability and fidelity results, we note that learning a surrogate on the coalitions of \textbf{Layer 1} not only offers complete stability and interesting fidelity, but it is also very fast to compute because there are only $2M$ of such coalitions (where $M$ is the total number of features). This makes Layer 1 a potentially interesting choice when learning attribution scores. In the following, we conduct a theoretical study of the attribution scores obtained with such a strategy.

        \section{First Layer Attribution Analysis}
        \label{sec:contrib3}
        
    In this section, we investigate the attribution values obtained using exclusively the coalitions from Layer 1.
    These attribution values are of particular interest because they are the cheapest to compute and guarantee stable explanations. 
    We first characterize these attribution values, and demonstrate their good theoretical properties, then we show to which extent they align with the exact SHAP values on real datasets, and finally we compare the execution times.  

    \subsection{Theoretical Analysis}

        In order to simplify the notations, we denote by $N=\{1,\cdots,M\}$ the set of all features and we define $f$ on $\mathcal{P}(N)$ as shown by Equation \eqref{eq:fext}.
        $\mathds{1}_S\in \{0,1\}^M$ is defined by $(\mathds{1}_S)_i=1$ if $i\in S$ and 0 otherwise, with $S$ being a coalition of features: 
        \begin{equation}
            \label{eq:fext}
            \begin{array}{rl}
                f: \mathcal{P}(N) &\rightarrow \mathbb{R}\\
                    S & \mapsto f_x({\mathds{1}}_S).
            \end{array}
        \end{equation}

        Coalitions of Layer 1 are of two types: coalitions with a single feature present (a single $1$) and coalitions with a single feature absent (a single $0$). In terms of marginal contribution, this amounts to averaging the individual contribution: $f(\{j\})-f(\emptyset)$, and the effect of removing the feature from the input: $f(N)-f(N\backslash \{j\}), \; \forall j \in N$.
    
        By also satisfying the {\em local accuracy} constraint \cite{lundberg2017unified}, i.e., $\sum^M_{i=1}{\phi_i} = f(x) - f(\emptyset)$,  we obtain Theorem~\ref{th:layer1_values}, which expresses a closed form for the attribution values of ST-SHAP when restricted to Layer 1. 

        \begin{theorem}[Attribution values with Layer 1]\label{th:layer1_values}
        For any feature $j$, the attribution value $\phi_j$ computed by ST-SHAP when filling exactly Layer 1 is: 
        \begin{equation}
                   \phi_j =  \tilde\phi_j
                     + \frac{1}{M}\left(f(N) - f(\emptyset)
                    - \sum_{i=1}^M \tilde\phi_i
                    \right)
                    \label{eq:valuePhi}
                \end{equation}
                where 
                for any $i$, $\tilde\phi_i=
                \frac{f(\{i\})- f(\emptyset) +
                f(N) - f(N\backslash\{i\})
                }{2}$.   
        \end{theorem}
        
         \begin{proof}
         We provide below a summary of the main steps of the proof, with some intermediate steps omitted. For a detailed proof of this theorem refer to \citet[Appendix~1]{kelodjou2024shapingupshap}.

       Determining the attribution values $\phi_j$ when restricting to coalitions of layer 1 is equivalent to solving the optimization problem formulated in Equation \eqref{eq:new_wlr}. The weighting kernel $\pi_x$ w.r.t Equation \eqref{eq:wlr} is omitted because all coalitions in layer 1 have the same weight. The formulation of the minimization function then becomes:
        \begin{equation}
        \textit{arg}\min_{\phi} \sum_{z'\in \{0,1\}^M}(g(z') - f_x(z'))^2.
        \label{eq:new_wlr}
        \end{equation}
        
        Coalitions $z'$ of layer 1 are characterized by $|z'| = 1$ or $|z'| = M-1$.
        \begin{itemize}
            \item When $|z'|=1$ with $j \in N$ being the only present feature, we obtain the equation:  \[g(z') = \phi_0 + \phi_j.\] 
            We aim to determine $g$ such that $g(z')$ provides a close approximation to $f(\{j\})$, i.e., $g(z') \approx f_x(z')$ where $f_x(z')=f(\{j\})$.
            \item When $|z'| = M-1$, with $j \in N$ being the only absent feature, we get: \[ g(z') = \phi_0 +  \sum_{\substack{i=1\\i\neq j\\ \phantom{-}}}^M \phi_{i}.\]
            We aim to determine $g$ such that $g(z')$ provides a close approximation to $f(N\setminus\{j\})$, i.e., $g(z') \approx f_x(z')$ where $f_x(z')=f(N\setminus\{j\})$.
        \end{itemize}
        
        \noindent Given the following terms:
        \begin{itemize}
            \item $f(\{j\}) - f(\emptyset)$ denoted by $\Delta_{j\emptyset}$
            \item $f(N\setminus \{j\}) - f(\emptyset)$ denoted by  $\Delta_{\bar{j}\emptyset}$
            \item $f(N) - f(N\setminus \{j\})$ denoted by $\Delta_{N\bar{j}}$
            \item $f(N) - f(\emptyset)$ denoted by $\Delta$,
        \end{itemize}
        \noindent this is tantamount to solve a minimization problem on the $\phi_j$ scores with two constraints: local accuracy ( $\sum_{i=0}^M\phi_i = f(N)$), and missingness ($\phi_0 = f(\emptyset)$). The minimization problem is defined by:
        \begin{equation}
        \textit{arg}\min_{\phi} \mathcal{L}(\bm{\phi})
        \end{equation}
        where $\mathcal{L}(\bm{\phi})$ is defined by:
        \begin{equation}
            \scalemath{0.8}{
        \mathcal{L}(\bm{\phi}) = \sum_{j=1}^M \left(\phi_j- \Delta_{j\emptyset}\right)^2 + \sum_{j=1}^M{\left(\left(\sum_{\substack{i=1\\i\neq j\\ \phantom{-}}}^M \phi_{i}\right)- \Delta_{\bar{j}\emptyset}\right)^2}
        }
        \end{equation}
        \begin{equation}
            \scalemath{0.8}{
                = \sum_{j=1}^M \left(2\phi_j^2 - 2(\Delta_{N\bar{j}} + \Delta_{j\emptyset})\phi_j + \Delta_{N\bar{j}}^2 + \Delta_{j\emptyset}^2  \right).
        }
        \end{equation}
        
        \noindent Because of the additional constraints (local accuracy), solving this problem requires using the method of the Lagrange multiplier, which leads us the following alternative optimization problem:
        \begin{equation}
            \textit{arg}\min_{\bm{\phi}, \lambda} \mathcal{L'}(\bm{\phi}, \lambda) = \mathcal{L}(\bm{\phi}) + \lambda(\sum_{j=1}^{M}{\phi_j} - \Delta).
        \end{equation} 
        This function is minimized when 
        \[
           \frac{\partial \mathcal{L}'}{\partial \bm{\phi}} = 0 \;\;\; \text{and} \;\;\; \frac{ \partial \mathcal{L}'}{\partial \lambda} = 0.
        \]
        For a given $\phi_j$ the new equations are:

        \begin{align}
        &\frac{\partial L'}{\partial \phi_j} = 0 = 4\phi_j -2\Delta_{N\bar{j}} -2\Delta_{j\emptyset} + \lambda \label{eq:phi}
        \\          
        & \frac{\partial L'}{\partial \lambda} = 0 = \Delta - \sum_{j=1}^M{\phi_j}.
       \end{align}
        
        \noindent Adding up the $M$ variants of Equation~\eqref{eq:phi}, we can solve for $\lambda$:
        
        \begin{equation}          
                 \lambda = \frac{\sum_{i=1}^M{(2\Delta_{N\bar{i}} + 2\Delta_{i\emptyset} ) - 4\sum_{i=1}^M{\phi_i}}}{M} \label{eq:lambda}.
        \end{equation}
        We can now plug Equation~\eqref{eq:lambda} into Equation~\eqref{eq:phi}, leading to our definition of $\phi_j$ for any $j\in N$: 
        \begin{equation}
           \phi_j =  \tilde\phi_j
                + \frac{1}{M}\left(f(N) - f(\emptyset)
                - \sum_{i=1}^M \tilde\phi_i
                \right).
        \end{equation}
         \end{proof}
        
        It can be noted that with such a formula, the three properties: Linearity, Symmetry, and Efficiency are verified \cite[Appendix~1]{kelodjou2024shapingupshap}. Therefore, the proposed attribution method belongs to the \textsc{LES} family of attribution scores ~\cite{ruiz1998family, chameni2008linear, radzik2013family}, which also includes the Shapley values. \textsc{LES} values are based on marginal contributions, providing feature contributions and interpretations very close to the Shapley values~\cite{condevaux2022fair}.

            \subsection{Layer 1 versus Shapley Values}

            We observed that the attribution values derived solely from the coalitions within Layer 1 are more straightforward to calculate and fulfill the three axioms of the LES family of attribution scores, also called LES values \cite{ruiz1998family, condevaux2022fair}. We now compare these values with the actual SHAP values (computed by materializing all coalitions) using various datasets. This comparison is based on the following considerations:
            \begin{itemize}[leftmargin=*]
                \item We include all the features in the explanation, meaning that the explanation size is $M$, i.e., the number of features in the target dataset. 
                \item We use the Kendall rank correlation coefficient to compare the rankings of the features within the explanation. This coefficient ranges from -1 to 1. Values close to $1$ indicate that the Layer 1's attribution ranks are in agreement with the actual SHAP values.
                \item To compare the disparity between the actual SHAP values and the Layer 1 attribution scores, we use the coefficient of determination, $R^2$. While mainly used to evaluate wellness of fit for machine learning models, we remark that it is suitable in our context. If we consider the SHAP values as the ``real'' values and the Layer 1 scores as the ``predicted'' values designed to fit the real values, the $R^2$ measures the proportion of the variance of the real SHAP values ``explained'' by the Layer 1 attribution scores. When the $R^2$ coefficient is 1 then both sets of scores agree, whereas values closer to 0 suggest that our approximation is as good as the naive attribution that assigns the same value to all features. Values smaller than 0 denote a performance below the naive baseline.
            \end{itemize}

            The experiments of this section are conducted on a subset of the datasets discussed in the \textit{Experiments} section, where computing the exact SHAP values is feasible. For black-box models, we employ the same models as described in the \textit{Experiments} section. The experiments are conducted on the whole test set (except for Credit card dataset, in which only $10$ random instances on the test set are used, due to high execution time required to compute the actual SHAP values) and we report averages and medians. 
            Table \ref{tab:Layer1_vs_Shapley} reports the corresponding results. We observe from Kendall's coefficient that the features ranks are almost the same between Layer 1's attribution scores and the exact SHAP values. The magnitudes of those scores also remain very similar as shown by the $R^2$ scores in a majority of cases. This makes Layer 1 attribution scores a very appealing choice for very large datasets: they are stable and very fast to compute, while still agreeing with the exact SHAP values.

           \begin{table}[t]
            \centering
           \fontsize{9}{10}\selectfont
           \addtolength{\tabcolsep}{-1pt}
                \begin{tabular}{@{}llccccrr@{}}
            \toprule
                                      &       & \multicolumn{3}{c}{Kendall $\tau$} &   \multicolumn{3}{c}{$R^2$-Score} \\  \cmidrule(lr){3-5}\cmidrule(lr){6-8}
                                      &        & SVM       & RF        & MLP        &   SVM       & RF       & MLP      \\ \midrule
            \multirow{2}{*}{Boston}   & $\mu$   & $0.95$    & $0.91$    & $0.959$    &   $0.98$    & $0.99$   & $0.99 $  \\
                                      & Med & $0.97$    & $0.92$    & $0.97$     &   $0.99$    & $0.99$   & $0.99$   \\ \midrule
            \multirow{2}{*}{Adult}    & $\mu$   & $0.7 $    & $0.68$    & $0.87$     &   $1.0 $    & $0.65$   & $0.41 $  \\
                                      & Med & $0.6$     & $ 0.7$    & $0.9$      &   $1.0$     & $0.77$   & $0.69$   \\ \midrule
            \multirow{2}{*}{Dry Bean} & $\mu$   & $0.86$    & $0.75$    & $0.84$     &   $-0.07$   & $0.74$   & $0.75$   \\
                                      & Med & $0.9$     & $0.76$    & $0.88$     &   $0.51$    & $0.79$   & $0.8$    \\ \midrule
            \multirow{2}{*}{Movie}    & $\mu$   & $1.0$     & $0.89$    & $0.85$     &   $1.0$     & $0.99$   & $0.95$   \\
                                      & Med & $1.0$     & $0.96$    & $0.86$     &   $1.0$     & $0.99$   & $0.95$ \\  \midrule
            \multirow{2}{*}{Credit Card}   & $\mu$   & $0.84$    & $0.64$    & $0.79$       & $0.94$    & $0.02$   & $0.81$   \\
                          & Med & $0.89$    & $0.71$    & $0.81$     &   $0.98$    & $0.21$   & $0.83$   \\
                         \bottomrule             
            \end{tabular}
            \caption{Kendall's $\tau$ and $R^2$-Score between the SHAP values and the attribution scores of ST-SHAP from the coalitions of Layer 1. $\mu$: Mean, Med: Median.}
	\label{tab:Layer1_vs_Shapley}
            \end{table}

        \subsection{Execution Times}
        The runtime of ST-SHAP Layer 1 is significantly smaller, compared to the computation of exact SHAP values, with complexities of $O(M)$ and $O(2^M)$ respectively. 

        Considering that the standard Kernel SHAP is recommended with a budget of $2000$ ($2\cdot M + 2^{11}$), utilizing ST-SHAP Layer 1 shows to be up to one order of magnitude faster than Kernel SHAP\footnote{It is also up to three orders of magnitude faster than computing the exact SHAP values.}. This is so because the budget is significantly lower than $2000$ as outlined in \cite[Appendix~2]{kelodjou2024shapingupshap}.        
	\section{Conclusion} \label{sec:conclusion}

 In this paper, we have investigated the instability issues of the Kernel SHAP estimator and proposed a novel neighbor selection approach that achieves full stability.
 Our experimental results demonstrate that our approach does not compromise fidelity. We have also conducted a theoretical analysis of the coefficients obtained by applying ST-SHAP on the neighbors of Layer 1. We define those scores formally and show experimentally that they remain very close to the actual SHAP values. This makes this approach an attribution method in itself, that incurs faster computation while retaining most of the desirable properties of attribution scores.
 
 In the future we are interested in understanding the properties of black-box models that guarantee good approximations of the SHAP values when trained on subsets of the coalitions space. This could include the definition of other attribution methods, and the relationship between the budget and the complexity of the target black-box we aim to explain.

\bibliography{reference}

\clearpage
\onecolumn
\appendix

\section{Appendix 1}
In this section, we prove Theorem 1 and show that the attribution values calculated using Layer 1 coalitions satisfy the properties of linearity, symmetry, and efficiency.

\subsection{Proof of Theorem 1}
In this Section, we prove Theorem 1, which we express again hereafter:    

\begin{theorem*}[Attribution values with Layer 1]
For any feature $j$, the attribution value $\phi_j$ computed by ST-SHAP when filling exactly the Layer 1 is:
\begin{align*}
   \phi_j &=  \tilde\phi_j
        + \frac{1}{M}\left(f(N) - f(\emptyset)
        - \sum_{i=1}^M \tilde\phi_i
        \right),
\end{align*}
        where 
        for any $i$, $\tilde\phi_i=
        \frac{f(\{i\})- f(\emptyset) +
        f(N) - f(N\backslash\{i\})
        }{2}$, $f$ the black-box model, $N = \{1, 2, \cdots, M\}$ the set of features, and $M$ the total number of features. 
\end{theorem*}

\begin{proof}
Consider $g$ as the linear model of binary features trained with coalitions, where the learned coefficients represent the attribution values of the corresponding features. Determining the attribution values $\phi_j$ when restricting to coalitions of layer $1$ is equivalent to solving the following optimization problem\footnote{The weighting kernel $\pi_x$ is omitted because all coalitions in layer $1$ have equal weight.}: 
        \begin{equation}
        \textit{arg}\min_{\phi} \; \;  \sum_{\mathclap{z'\in \{0,1\}^M}} \; \; \; {\Big(g(z') - f_x(z')\Big)^2}.
        \label{eq:new_wlr2}
        \end{equation}
        
        {\vspace{0.3cm} Coalitions $z'$ of layer $1$ are characterized by $|z'| = 1$ or $|z'| = M-1$.}
        
        \begin{itemize}
            \item When $|z'|=1$ with $j \in N$ being the only present feature, we obtain the equation:  \[g(z') = \phi_0 + \phi_j.\] 
            We aim to determine $g$ such that $g(z')$ provides a close approximation to $f(\{j\})$, i.e., $g(z') \approx f_x(z')$ where $f_x(z')=f(\{j\})$.
            \item When $|z'| = M-1$, with $j \in N$ being the only absent feature, we get: \[ g(z') = \phi_0 +  \sum_{\substack{i=1\\i\neq j\\ \phantom{-}}}^M \phi_{i}.\]
            We aim to determine $g$ such that $g(z')$ provides a close approximation to $f(N\setminus\{j\})$, i.e., $g(z') \approx f_x(z')$ where $f_x(z')=f(N\setminus\{j\})$.
        \end{itemize}
        
        \noindent Given the following terms:
        \begin{itemize}
            \item $f(\{j\}) - f(\emptyset)$ denoted by $\Delta_{j\emptyset}$
            \item $f(N\setminus \{j\}) - f(\emptyset)$ denoted by  $\Delta_{\bar{j}\emptyset}$
            \item $f(N) - f(N\setminus \{j\})$ denoted by $\Delta_{N\bar{j}}$
            \item $f(N) - f(\emptyset)$ denoted by $\Delta$,
        \end{itemize}
        \noindent this is tantamount to solving a minimization problem on the $\phi_j$ scores with two constraints: 
        \begin{itemize}
            \item $\phi_0 = f(\emptyset)$
            \item $\sum_{i=1}^M\phi_i = \Delta$.
        \end{itemize}
        The minimization problem is defined by:
    \begin{equation}
    \textit{arg}\min_{\phi} \mathcal{L}(\bm{\phi}),     
    \end{equation}
where $\mathcal{L}(\bm{\phi})$ is defined by:
    \begin{align}
        &\mathcal{L}(\bm{\phi}) = \sum_{j=1}^M{\Big( (\phi_0 + \phi_j) - f(\{j\}) \Big)^2} + \sum_{j=1}^M{\left( \left( \phi_0 + \sum_{\substack{i=1\\i\neq j\\ \phantom{-}}}^M{\phi_i}\right) - f(N\setminus \{j\}) \right)^2}
        \\
        &= \sum_{j=1}^M{\left(\phi_j- \Delta_{j\emptyset}\right)^2} + \sum_{j=1}^M{\left(\left(\sum_{\substack{i=1\\i\neq j\\ \phantom{-}}}^M \phi_{i}\right)- \Delta_{\bar{j}\emptyset}\right)^2} \;\;\;[\text{Due to the requirement} \;\; \phi_0 = f(\emptyset)]
        \\ 
         &= \sum_{j=1}^M (\phi_j- \Delta_{j\emptyset})^2 + \sum_{j=1}^M (\Delta - \phi_j- \Delta_{\bar{j}\emptyset} )^2 \;\;\;[\text{Due to the requirement} \;\; \sum_{i=1}^M{\phi_i} = \Delta]
        \\
        &=
        \sum_{j=1}^M (\phi_j- \Delta_{j\emptyset})^2 + \sum_{j=1}^M (\Delta_{N\bar{j}} - \phi_j)^2 \;\;\;[\text{Because}\;\; \Delta - \Delta_{\bar{j}\emptyset} = \Delta_{N\bar{j}}] 
        \\
        &= \sum_{j=1}^M \left(2\phi_j^2 - 2(\Delta_{N\bar{j}} + \Delta_{j\emptyset})\phi_j + \Delta_{N\bar{j}}^2 + \Delta_{j\emptyset}^2  \right). 
\end{align}

\noindent The minimization of $\mathcal{L}$ is constrained by \(\sum_{j=1}^{M}{\phi_j} - \Delta = 0\). To solve this, we can use the method of the Lagrange multiplier i.e., we minimize instead:

\begin{equation}
    \textit{arg}\min_{\bm{\phi}, \lambda} \mathcal{L'}(\bm{\phi}, \lambda) = \mathcal{L}(\bm{\phi}) + \lambda\Big(\sum_{j=1}^{M}{\phi_j} - \Delta\Big).
\end{equation} 

This is solved when 

\[
   \frac{\partial \mathcal{L}'}{\partial \bm{\phi}} = 0 \;\;\; \text{and} \;\;\; \frac{ \mathcal{L}'}{\partial \lambda} = 0.
\]

If we rearrange the terms we can see that $\mathcal{L'}(\bm{\phi}, \lambda)$ has the following form:

\begin{equation}
    \mathcal{L'}(\bm{\phi}, \lambda) =  \sum_{j=1}^M \Big( 2\phi_j^2 - (2\Delta_{N\bar{j}} + 2\Delta_{j\emptyset} - \lambda)\phi_j + \Delta_{N\bar{j}}^2 + \Delta_{j\emptyset}^2  \Big) - \lambda\Delta.
\end{equation}
Hence,

\begin{align}
        &\frac{\partial \mathcal{L'}}{\partial \phi_j} = 0 = 4\phi_j -2\Delta_{N\bar{j}} -2\Delta_{j\emptyset} + \lambda\;\; \therefore \;\; 4\phi_j = 2\Delta_{N\bar{j}} + 2\Delta_{j\emptyset} - \lambda \label{eq:phi2}
        \\          
        & \frac{\partial \mathcal{L'}}{\partial \lambda} = 0 = \Delta - \sum_{j=1}^M{\phi_j}. 
        \\ 
\end{align}
Adding up the $M$ variants of Equation~\eqref{eq:phi2}, we can solve for $\lambda$:

\begin{align}
        &4\sum_{i=1}^M{\phi_i} = \sum_{i=1}^M{(2\Delta_{N\bar{i}} + 2\Delta_{i\emptyset})} - M\lambda 
        \\          
        & \lambda = \frac{\sum_{i=1}^M{(2\Delta_{N\bar{i}} + 2\Delta_{i\emptyset} ) - 4\sum_{i=1}^M{\phi_i}}}{M}. \label{eq:lambda2}
\end{align}
We can now plug Equation~\eqref{eq:lambda2} into Equation~\eqref{eq:phi2}:
\begin{align}
        &4\phi_j = 2\Delta_{N\bar{j}} + 2\Delta_{j\emptyset} - \left( \frac{\sum_{i=1}^M{(2\Delta_{N\bar{i}} + 2\Delta_{i\emptyset} ) - 4\sum_{i=1}^M{\phi_i}}}{M} \right)
        \\
         &4\phi_j = 2\Delta_{N\bar{j}} + 2\Delta_{j\emptyset} + \frac{1}{M} \left( 4\Delta - \sum_{i=1}^M{ 2\Delta_{N\bar{i}} + 2\Delta_{i\emptyset}}\right)
        \\
        &\phi_j = \frac{\Delta_{N\bar{j}}}{2}+ \frac{\Delta_{j\emptyset}}{2} + \frac{1}{M} \left( \Delta - \sum_{i=1}^M{ \frac{ \Delta_{N\bar{i}}+ \Delta_{i\emptyset}}{2}}\right). 
        \\
\end{align}

meaning for any $j\in N$
\begin{align}
   \phi_j &=  \tilde\phi_j
        + \frac{1}{M}\left(f(N) - f(\emptyset)
        - \sum_{i=1}^M \tilde\phi_i
        \right).
\end{align}

\end{proof}

\subsection{\textsc{LES} properties of Layer 1 Explanations}
Our method for approximating Shapley values --by calculating attribution values using only Layer $1$ coalitions-- belongs to the \textsc{LES} family, thereby satisfying the properties of Linearity, Symmetry, and Efficiency. We demonstrate here that these three fundamental properties are verified.

 \subsubsection{Linearity.} The attribution method based on Layer $1$ coalitions is linear if for any models $f_1, f_2$ with an instance $x$, and for any $\alpha_1, \alpha_2 \in \mathbb{R}$: 
 \[
 \phi_j(x, \alpha_1 f_1 + \alpha_2 f_2) = \alpha_1 \phi_j(x, f_1) + \alpha_2 \phi_j(x, f_2).
 \]
 We can proceed in two steps: (i) prove that $\phi_j(x, \alpha f) = \alpha \phi_j(x, f)$ and then (ii) $\phi_j(x, f_1 + f_2) = \phi_j(x, f_1) + \phi_j(x, f_2)$.

\begin{equation*}
        \phi_j =  \tilde\phi_j
                 + \frac{1}{M}\left(f(N) - f(\emptyset)
                 - \sum_{i=1}^M \tilde\phi_i
                    \right)
\end{equation*}

where for any $i \in N$ (the set of features), 
\begin{equation*}
    \tilde\phi_i=
                \frac{f(\{i\})- f(\emptyset) + f(N) - f(N\backslash\{i\})}{2}.
\end{equation*}

    \begin{itemize}

        \item First, we want to prove that if $g=\alpha f$ then $\forall j \in N$, $\phi_j(x,g)=\alpha \phi_j(x,f)$.
        
    Assume \(g=\alpha f\) and take \(j \in N\) and \(i \in N\), then:
        \begin{align*}
            \tilde\phi_i(x,g)&=
                \frac{g(\{i\})- g(\emptyset) + g(N) - g(N\backslash\{i\})}{2}\\
            &= \frac{\alpha f(\{i\})- \alpha f(\emptyset) + \alpha f(N) - \alpha f(N\backslash\{i\})}{2} \\
            &= \alpha \phi_i(x,f).
        \end{align*}

        \begin{align*}
        \phi_j(x,g) &=  \tilde \phi_j(x,g)+ \frac{1}{M}\left(g(N) - g(\emptyset)- \sum_{i=1}^M \tilde\phi_i(x,g) \right)\\
                    &=  \alpha \tilde\phi_j(x,f)+ \frac{1}{M}\left(\alpha f(N) - \alpha f(\emptyset)- \sum_{i=1}^M \alpha \tilde\phi_i(x,f)\right)\\
                    &=\alpha \left(\tilde\phi_j(x,f)+ \frac{1}{M}\left(f(N) -  f(\emptyset)- \sum_{i=1}^M  \tilde\phi_i(x,f)\right)\right)\\
                    &=\alpha \phi_j(x,f).
        \end{align*}

        \item Second, we want to prove that if \(g=f_1 + f_2\) then \(\forall j \in N\), \( \phi_j(g,x)=\phi_j(x,f_1)+\phi_j(x,f_2) \). \\
        Assume $g=f_1+f_2$ and take $j \in N$ and $i \in N$, then:

        \begin{align*}
            \tilde\phi_i(x,g)&=
                \frac{g(\{i\})- g(\emptyset) + g(N) - g(N\backslash\{i\})}{2}\\
            &= \frac{f_1(\{i\}) + f_2(\{i\})- (f_1(\emptyset)+f_2(\emptyset)) + (f_1(N)+f_2(N)) - (f_1(N\backslash\{i\})+f_2(N\backslash\{i\}))}{2} \\
            &= \frac{f_1(\{i\}) - f_1(\emptyset) + f_1(N) - f_1(N\backslash\{i\})}{2} + \frac{f_2(\{i\}) - f_2(\emptyset) + f_2(N) - f_2(N\backslash\{i\}))}{2} \\
            &= \tilde\phi_i(x,f_1)+\tilde\phi_i(x,f_2).
        \end{align*}

               \begin{align*}
        \phi_j(x,g) &=  \tilde \phi_j(x,g)+ \frac{1}{M}\left(g(N) - g(\emptyset)- \sum_{i=1}^M \tilde\phi_i(x,g) \right)\\
                    &=  \tilde\phi_j(x,f_1)+\tilde\phi_j(x,f_2)+ \frac{1}{M}\left(f_1(N)+f_2(N) - (f_1(\emptyset)+f_2(\emptyset))- \sum_{i=1}^M (\tilde\phi_i(x,f_1)+\tilde\phi_i(x,f_1))\right)\\
                    &=\tilde\phi_j(x,f_1)+ \frac{1}{M}\left(f_1(N) -  f_1(\emptyset)- \sum_{i=1}^M  \tilde\phi_i(x,f_1)\right)
                        + \tilde\phi_j(x,f_2)+ \frac{1}{M}\left(f_2(N) -  f_2(\emptyset)- \sum_{i=1}^M  \tilde\phi_i(x,f_2)\right)\\
                    &=\phi_j(x,f_1) + \phi_j(x,f_2).
        \end{align*}

    \end{itemize}
\vspace{.2cm}
    \subsubsection{Efficiency.} We want to prove that  $\sum_{i=1}^M \phi_i = f(N) - f(\emptyset)$
  \begin{align*}
        \phi_j &=  \tilde\phi_j+ \frac{1}{M}\left(f(N) - f(\emptyset)- \sum_{i=1}^M \tilde\phi_i\right)\\
        \sum_{j=1}^M \phi_j &=  \sum_{j=1}^M \left(\tilde\phi_j+ \frac{1}{M}\left(f(N) - f(\emptyset)- \sum_{i=1}^M \tilde\phi_i\right)\right)\\
              &=\sum_{j=1}^M \tilde\phi_j +\left(f(N) - f(\emptyset)- \sum_{i=1}^M \tilde\phi_i\right)\\
              &=f(N) - f(\emptyset).
\end{align*}  

    \subsubsection{Symmetry.} We want to show that

\begin{equation}
    \left(
    \left.
    \begin{aligned}
        & \forall j, k \in N \\
        & \forall S \in \mathcal{P}(N), \; \textrm{with} \; j,k \notin S
    \end{aligned}
    \right\} \quad
    f(S\cup\{j\}) = f(S\cup\{k\}) \right) 
    \Rightarrow \phi_j = \phi_k
\end{equation}

\vspace{.2cm}
    \noindent Assume: 
    \begin{equation}
        \left.
        \begin{aligned}
        &\forall j,k \in N\\
        &\forall S  \in \mathcal{P}(N), \; \textrm{with} \; j,k \notin S
        \end{aligned} 
        \right\} \quad
        f(S\cup\{j\})=f(S\cup\{k\})
    \end{equation}
    Consider specific cases of $S$:
    \begin{itemize}
        \item For $S=\emptyset$, we have   \(f(\{j\})=f(\{k\}).\)
        \item For $S=N\backslash\{j,k\}$, we have
        \(f(N\backslash\{k\})=f(N\backslash\{j\}).\)
    \end{itemize}
    \vspace{.2cm}
  Thus, we have \(f(\{j\}) = f(\{k\})\) and \(f(N\backslash\{j\}) = f(N\backslash\{k\}).\)
  
\vspace{.2cm}
\noindent The Layer $1$ attribution values can be rewritten as follows
        \begin{equation}
            \phi_j=\frac{f(\{j\})-f(N\backslash\{j\})}{2} - \frac{\sum_{i=1}^M f(\{i\})-f(N\backslash\{i\})}{2M} +  \frac{f(N) - f(\emptyset)}{M}
        \label{eq:phiLayer1}
        \end{equation}

        % \vspace{.3cm}
   \noindent Since  $f(\{j\}) = f(\{k\})$ and $f(N\backslash\{j\}) = f(N\backslash\{k\})$, it follows that:
    
    \begin{equation*}
        f(\{j\})-f(N\backslash\{j\})=f(\{k\})-f(N\backslash\{k\}).
        \end{equation*}

    \vspace{.3cm}
    \noindent The other terms in the expression for $\phi_j$ do not depend on $j$ or $k$, so we conclude that $\phi_j=\phi_k$.
\clearpage
\section{Appendix 2}

\subsection{Additional experiments on comparing SHAP and ST-SHAP}
In this section, we present further results on the stability of Kernel SHAP (called \textsc{SHAP} here) and our stability-enhancing approach, \textsc{ST-SHAP}.

\begin{center}
    \Image[width=0.45\linewidth]{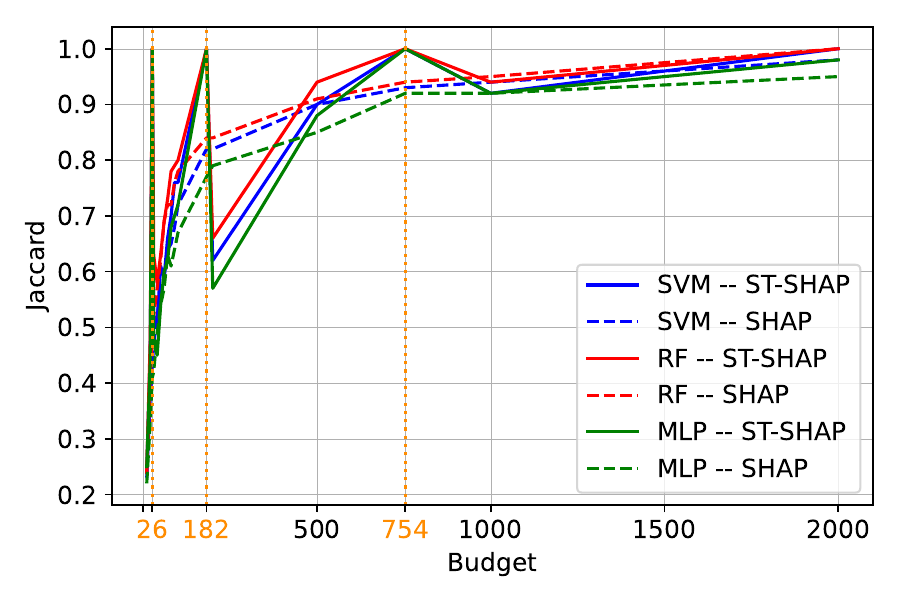}{(a) Jaccard} \,
    \Image[width=0.45\linewidth]{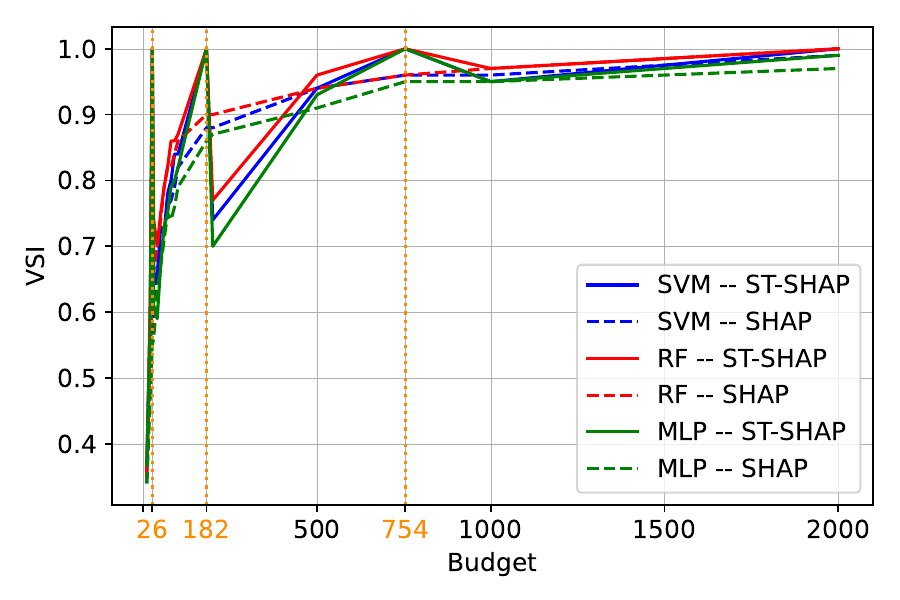}{(b) VSI} \,
    
    \Image[width=.406\linewidth]{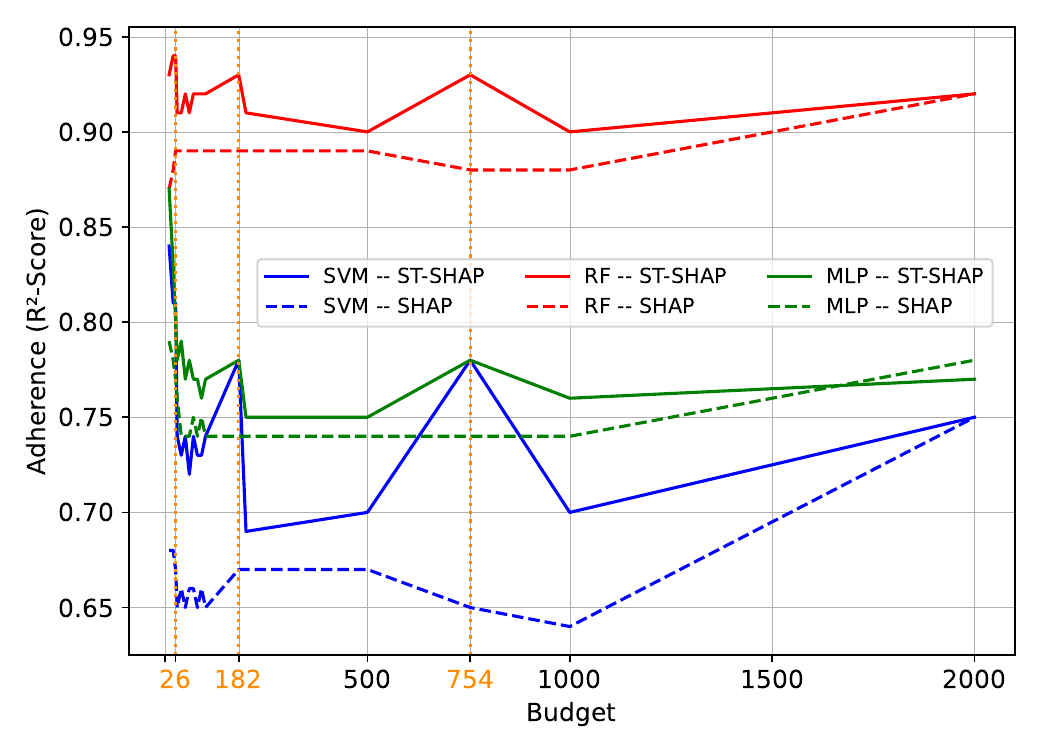}{(c) Adherence with the black-box} \,
    \Image[width=0.45\linewidth]{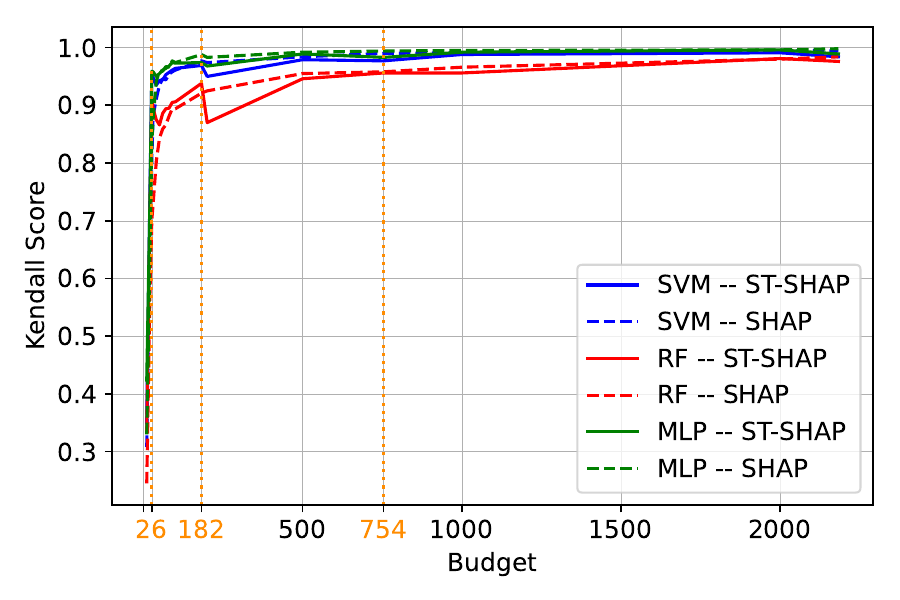}{(d) Kendall $\tau$ with the exact SHAP values} \,

    \Image[width=0.45\linewidth]{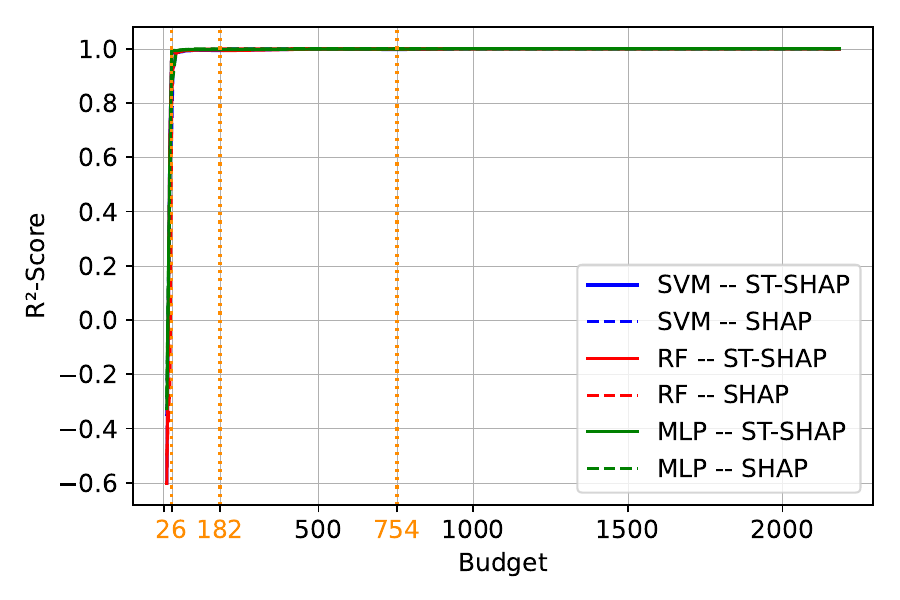}{(e) $R^2$-Score with the exact SHAP values} \,
    
    \captionof{figure}{Comparison results between SHAP and ST-SHAP across various criteria on the Boston dataset.}
    \vspace{.2cm}
    \label{fig:results_boston}
\end{center}

In Figure \ref{fig:results_boston}, we compare SHAP and ST-SHAP to evaluate their stability, fidelity, and the quality of their approximation of the exact SHAP values. Figures \ref{fig:results_boston}a and \ref{fig:results_boston}b depict stability outcomes assessed using two state-of-the-art metrics: the Jaccard index and the Variable Stability Index (VSI) \cite{visani2022statistical}, respectively. An initial observation highlights a significant similarity between these two plots, suggesting a certain equivalence between these metrics. 
This is the rationale behind our decision to utilize Jaccard's index to present the stability results.
Furthermore, it is noted that ST-SHAP achieves maximum stability (both Jaccard and VSI at $1.0$) across complete layers. Additionally, as the budget approaches completeness for a layer, stability also tends toward $1.0$ on ST-SHAP. Figure \ref{fig:results_boston}c, on the other hand, depicts the results concerning fidelity with respect to the black-box model. It is noticeable that ST-SHAP performs better than SHAP, with peaks occurring over complete layers. This leads us to conclude that our ST-SHAP methods do not impact the adherence of the explanations. Lastly, we compared the approximation quality of the exact SHAP values using the SHAP and ST-SHAP surrogates. Figure \ref{fig:results_boston}d illustrates the outcomes based on Kendall's coefficient, revealing a similar ranking between the two methods and a substantial agreement with the exact SHAP values. Regarding the disparities between the SHAP and ST-SHAP coefficients relative to the exact SHAP values, it is noticeable in Figure \ref{fig:results_boston}e that both methods align closely and provide effective approximations of the exact SHAP values.

\begin{center}
\vspace{.3cm}
    \Image[width=0.45\linewidth]{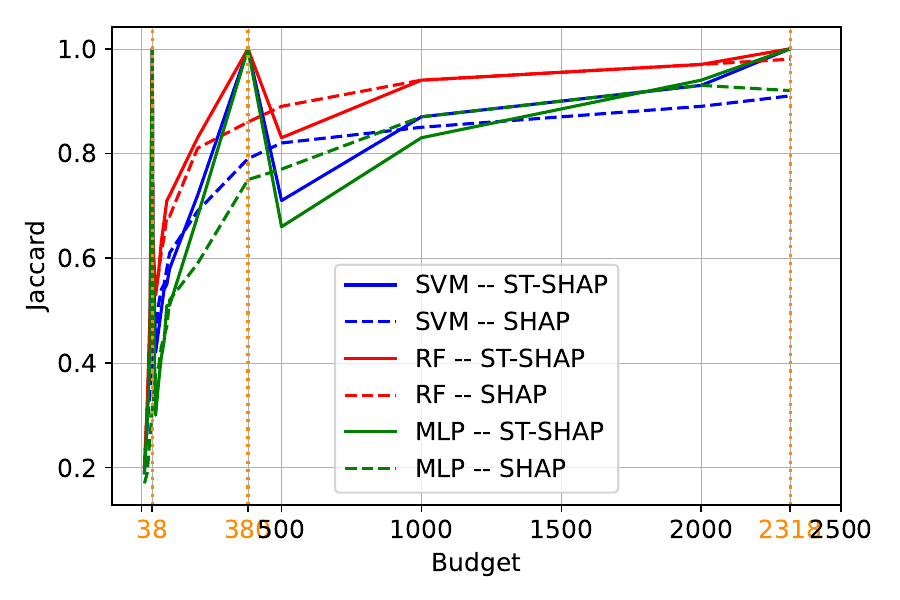}{(a) Jaccard} \,
    \Image[width=0.45\linewidth]{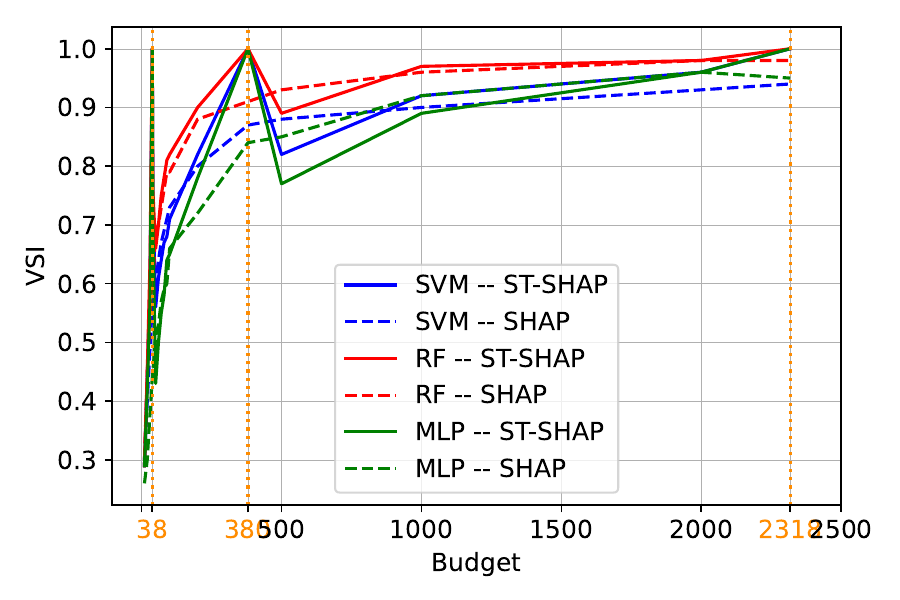}{(b) VSI} \,
    \vspace{.2cm}
    \Image[width=.406\linewidth]{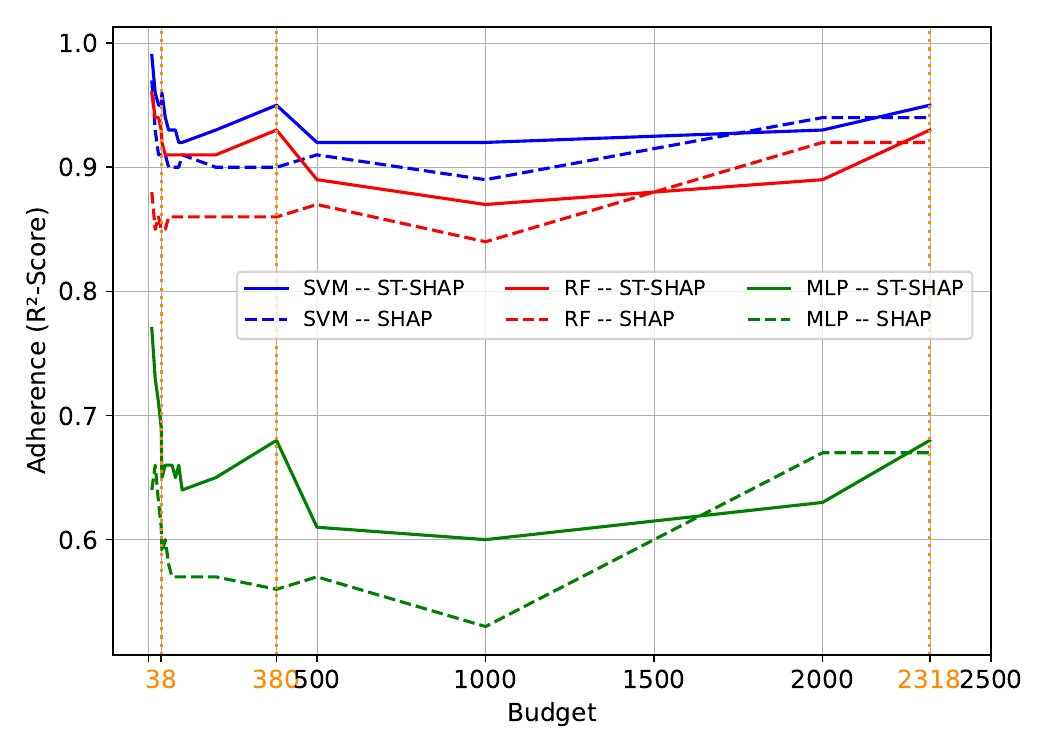}{(c) Adherence with the black-box} \,
    \Image[width=0.45\linewidth]{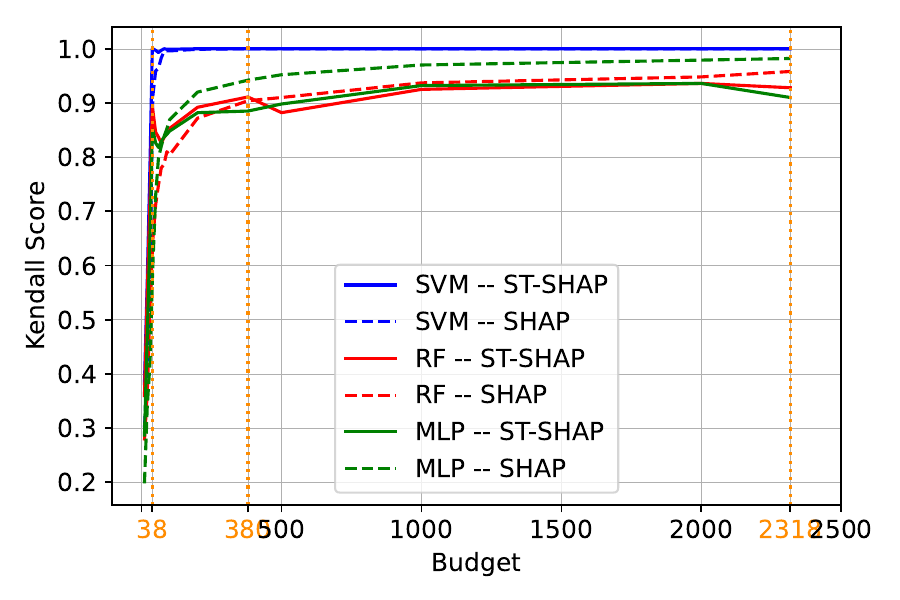}{(d) Kendall $\tau$ with the exact SHAP values} \,
    \vspace{.2cm}
    \Image[width=0.45\linewidth]{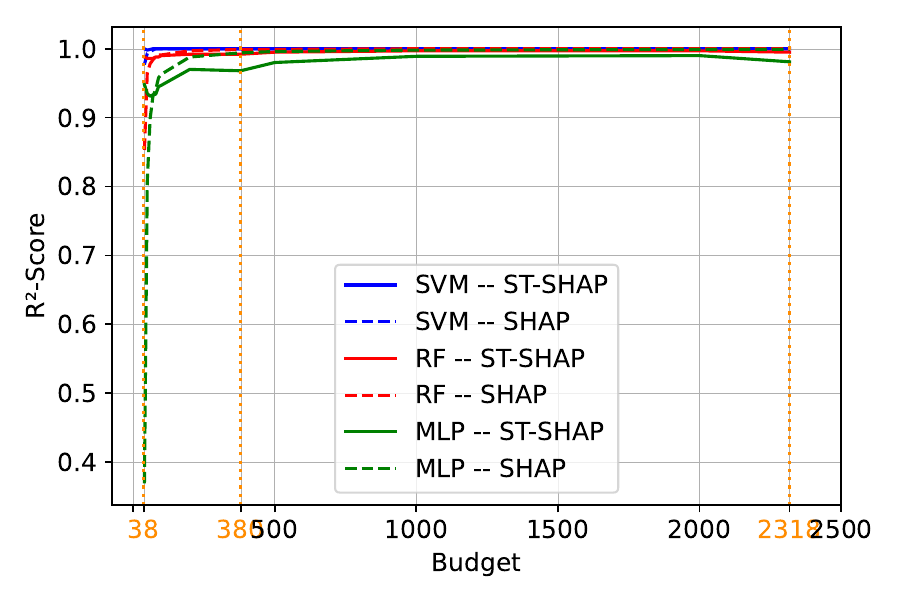}{(e) $R^2$-Score with the exact SHAP values} \,
   
    \captionof{figure}{Comparison results between SHAP and ST-SHAP across various criteria on the Movie dataset.}
    \vspace{.2cm}
    \label{fig:results_movie}
\end{center}

Figure \ref{fig:results_movie} illustrates SHAP and ST-SHAP results on the Movie dataset. Stability (Figures \ref{fig:results_movie}a and \ref{fig:results_movie}b) and fidelity (Figure \ref{fig:results_movie}c) criteria have the same observations as those on the Boston dataset. Regarding the precision in approximating the exact SHAP values, Figure \ref{fig:results_movie}d showcases a similar ranking, with SHAP slightly better than ST-SHAP. Regarding the attribution coefficients of both methods compared to the exact SHAP values, SHAP presents relatively weaker performance for smaller budgets (as depicted in Figure \ref{fig:results_movie}e).  However, overall, both curves exhibit similar performance levels.

\begin{center}
    \Image[width=0.45\linewidth]{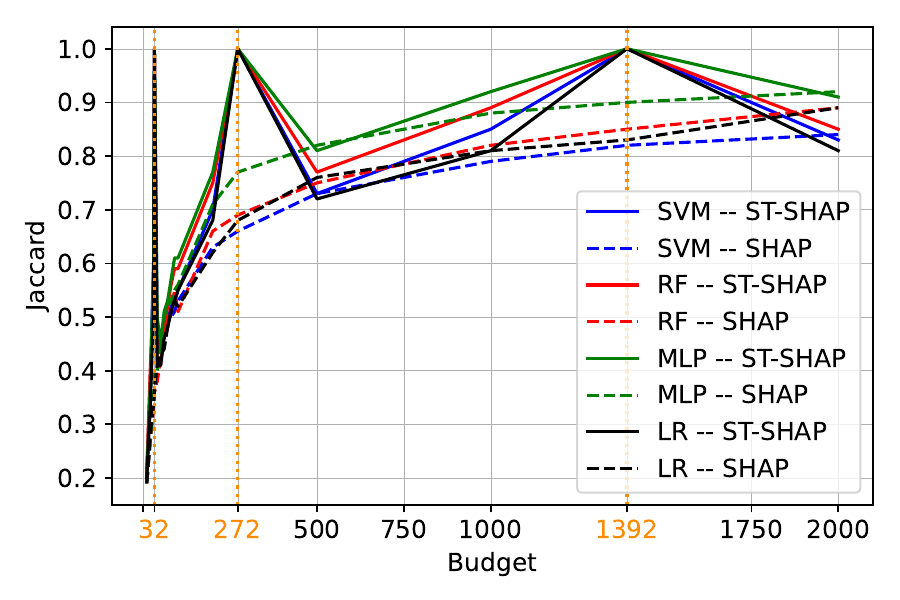}{(a) Jaccard} \,
    \Image[width=0.45\linewidth]{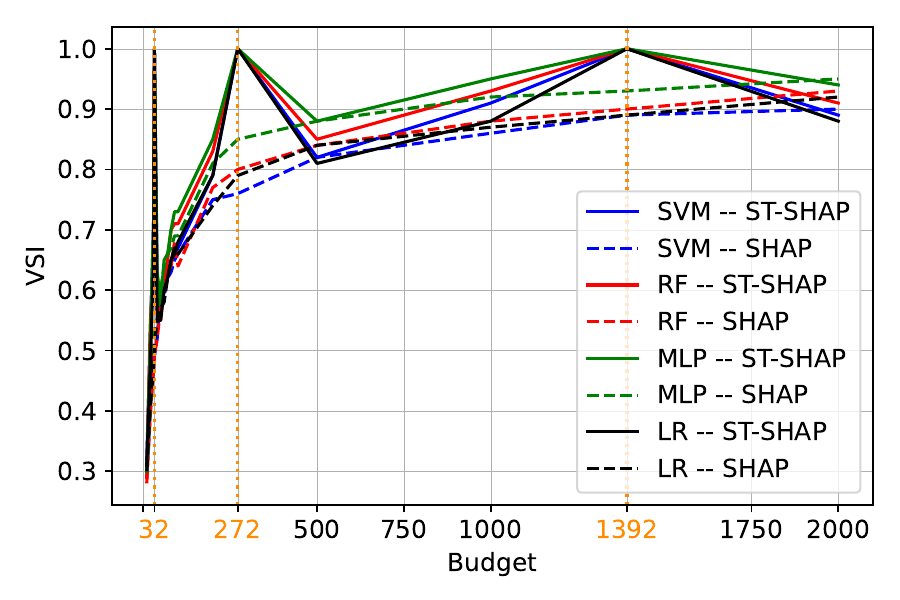}{(b) VSI} \,
    
    \vspace{.1cm}
    \Image[width=.406\linewidth]{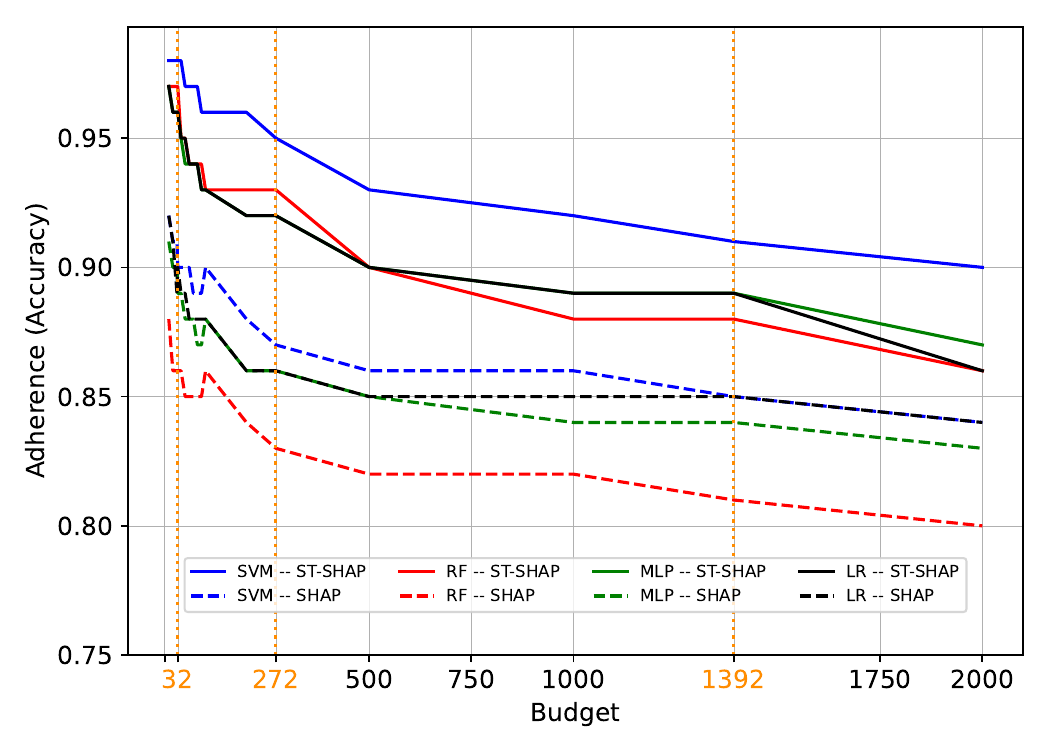}{(c) Adherence with the black-box} \,
    \Image[width=0.45\linewidth]{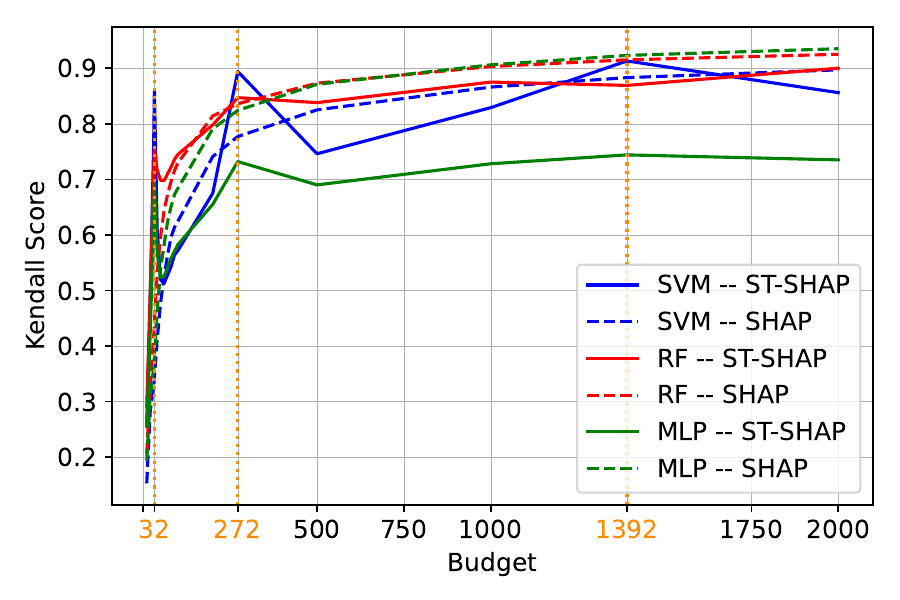}{(d) Kendall $\tau$ with the exact SHAP values} \,
    
    \vspace{.1cm}
    \Image[width=0.45\linewidth]{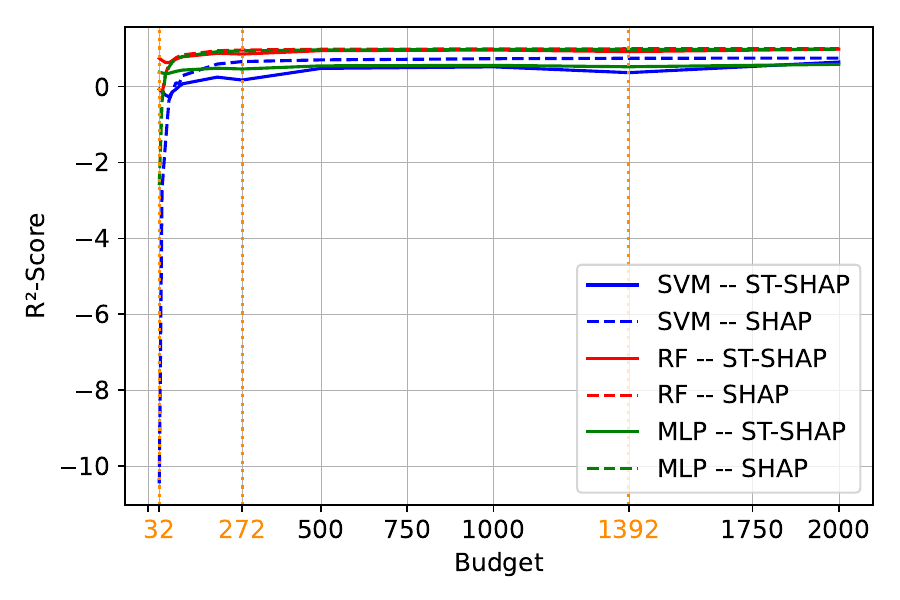}{(e) $R^2$-Score with the exact SHAP values} \,
    
    \captionof{figure}{Comparison results between SHAP and ST-SHAP across various criteria on the Dry Bean dataset.}
    \vspace{.2cm}
    \label{fig:results_drybean}
\end{center}

Figure \ref{fig:results_drybean} compares ST-SHAP and SHAP on the Dry Bean dataset. Figures \ref{fig:results_drybean}a, \ref{fig:results_drybean}b, and \ref{fig:results_drybean}c share the same interpretation as their counterparts in other datasets. As for the ranking quality (according to the exact SHAP values) using Kendall's coefficient in Figure  \ref{fig:results_drybean}d, a general similarity is noticed among all models, except for ST-SHAP with the MLP function, which exhibits a lower ranking score than SHAP with MLP. Although there are generally similar performances between the two methods on the $R^2$ score compared to the exact SHAP values (as shown in Figure \ref{fig:results_drybean}e), weak performances of SHAP in smaller budgets are observed, especially with the SVM model.

\begin{center}
    \Image[width=0.45\linewidth]{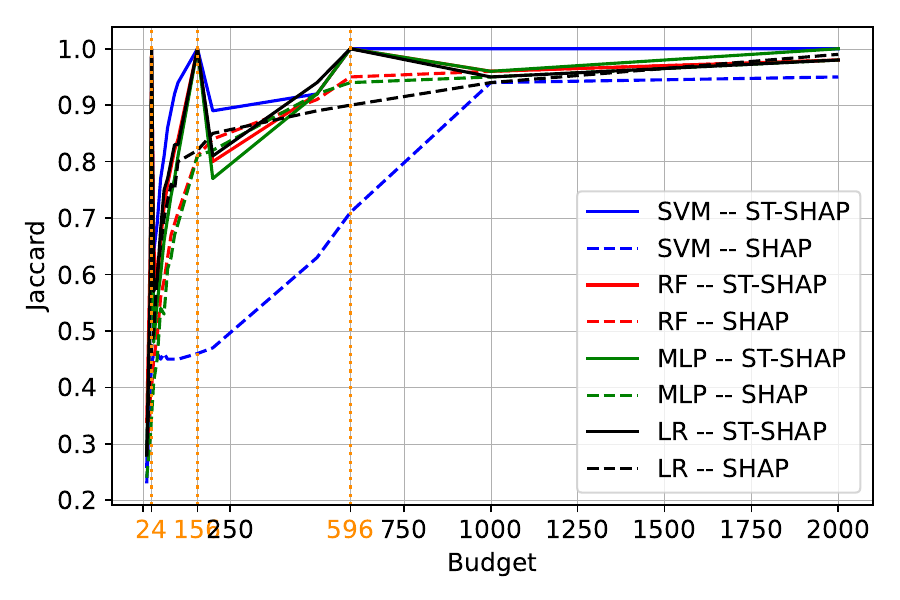}{(a) Jaccard} \,
    \Image[width=.406\linewidth]{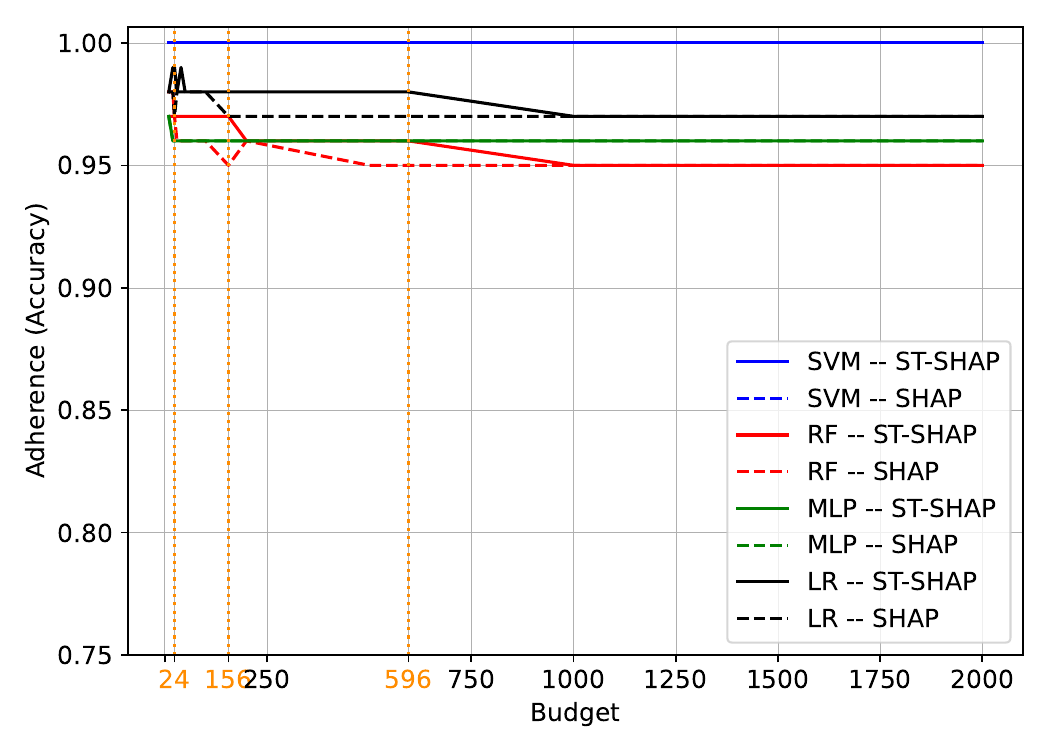}{(b) Adherence with the black-box} \,
    
    \Image[width=0.45\linewidth]{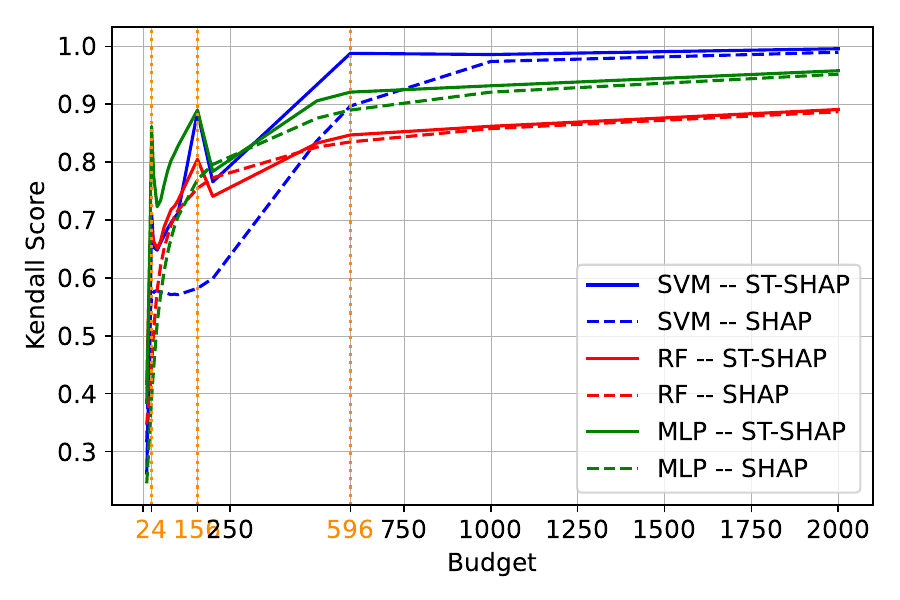}{(c) Kendall $\tau$ with the exact SHAP values} \,
    \Image[width=0.45\linewidth]{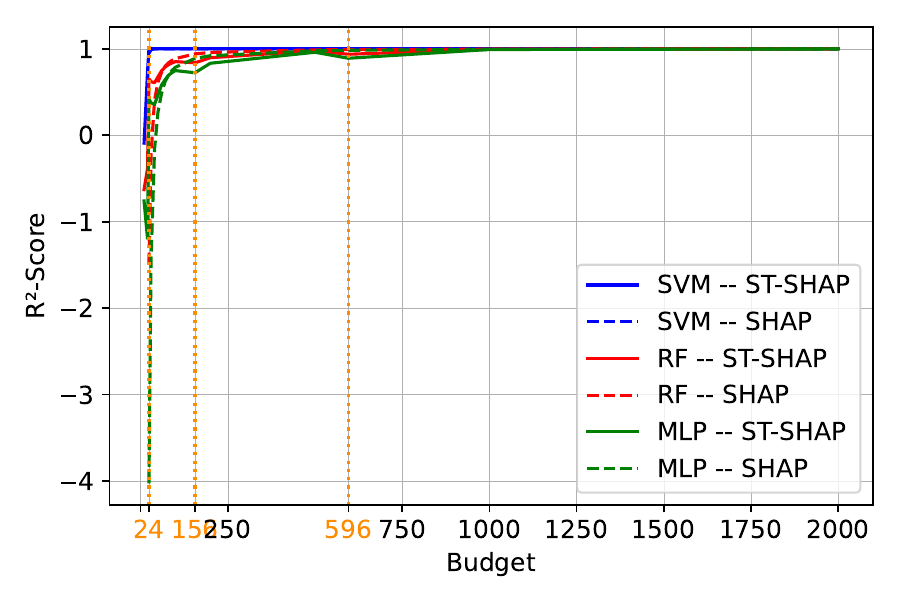}{(d) $R^2$-Score with the exact SHAP values} \,
    
    \captionof{figure}{Comparison results between SHAP and ST-SHAP across various criteria on the Adult Income dataset.}
    \vspace{.2cm}
    \label{fig:results_adult}
\end{center}

Figure \ref{fig:results_adult} illustrates SHAP and ST-SHAP results on the Adult Income dataset. ST-SHAP is generally better on stability than SHAP (Figure \ref{fig:results_adult}a on Jaccard Index). With adherence metrics ranging from $0.95$ to $1.0$, we can deduce that our approach ST-SHAP exhibits similar behavior to SHAP (Figure \ref{fig:results_adult}b on Adherence to black-box with $R^2$-Score). Figure \ref{fig:results_adult}c depicts the Kendall coefficient according to exact SHAP values, showing comparable performance between both methods, with weaker performance on SHAP with SVM. ST-SHAP slightly outperforms SHAP, particularly with peaks in complete layers.
Moving on to Figure \ref{fig:results_adult}d, we show the outcomes in terms of $R^2$-score concerning the exact SHAP values. The results generally exhibit similar trends, although SHAP demonstrates weaker performance in smaller budgets.

\begin{center}
\vspace{.2cm}
    \Image[width=0.45\linewidth]{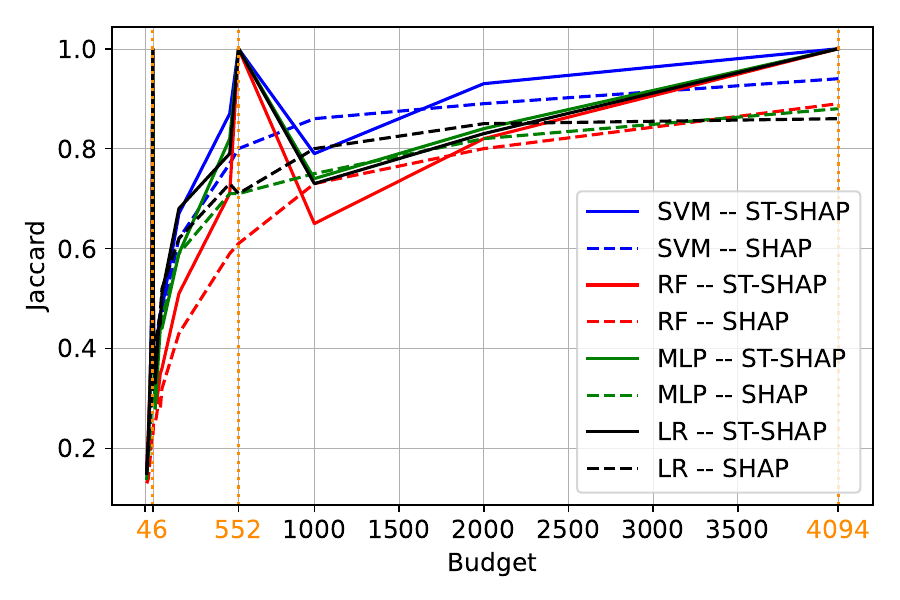}{(a) Jaccard} \,
    \Image[width=.406\linewidth]{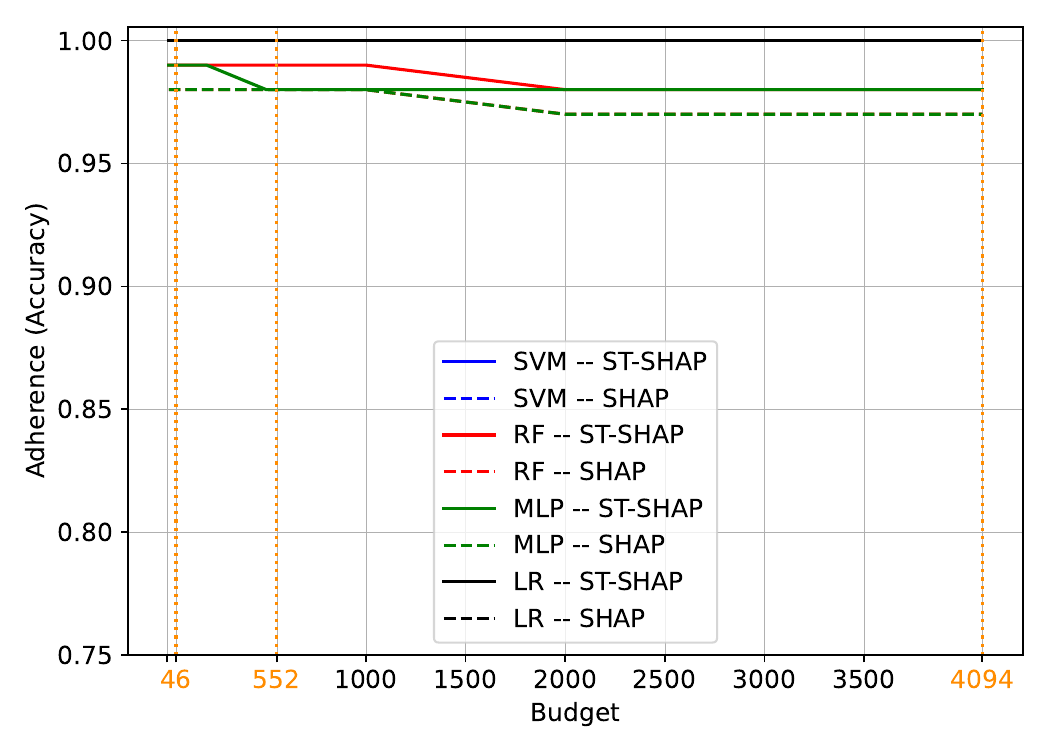}{(b) Adherence with the black-box} \,
    \captionof{figure}{SHAP and ST-SHAP on the Credit dataset.}
    \vspace{.2cm}
    \label{fig:results_credit}
\end{center}

 Figure \ref{fig:results_credit} presents Kernel SHAP and ST-SHAP results on the Credit dataset. For the stability criterion (Figure \ref{fig:results_credit}a), the interpretations made for the other datasets also apply to this one. As the adherence (Figure \ref{fig:results_credit}b) is between $0.97$ and $1.0$, we can conclude that our approach behaves similarly to SHAP, and both approaches exhibit good performance.

\begin{center}
    \Image[width=0.5\linewidth]{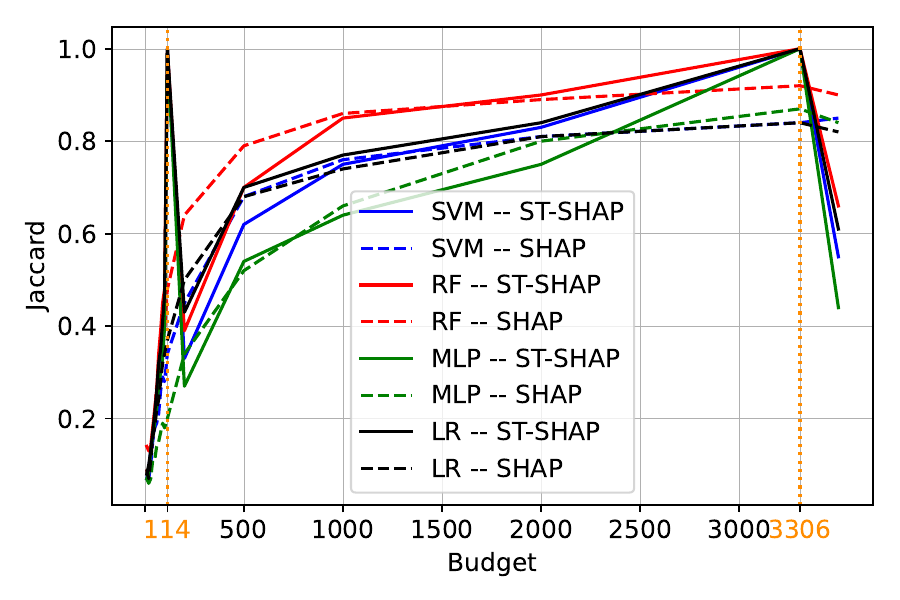}{(a) Jaccard} \,
    \Image[width=.456\linewidth]{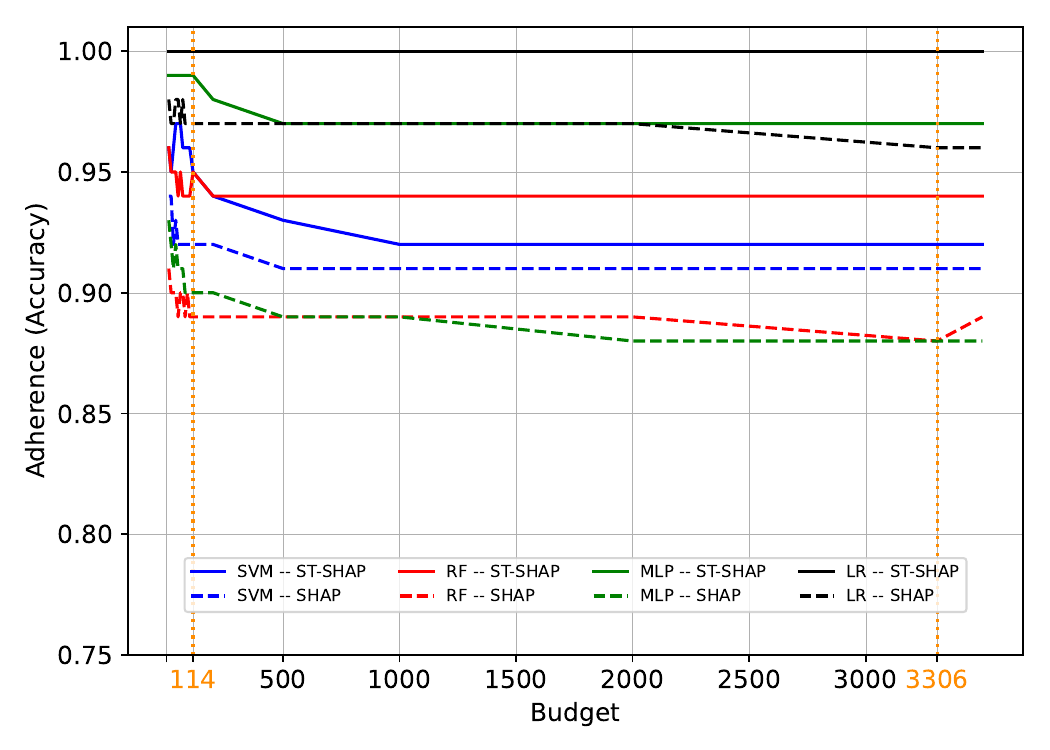}{(b) Adherence with the black-box} \,
    \captionof{figure}{SHAP and ST-SHAP on the Spambase dataset.}
    \vspace{.2cm}
    \label{fig:results_spambase}
\end{center}

 \begin{center}
    \Image[width=0.45\linewidth]{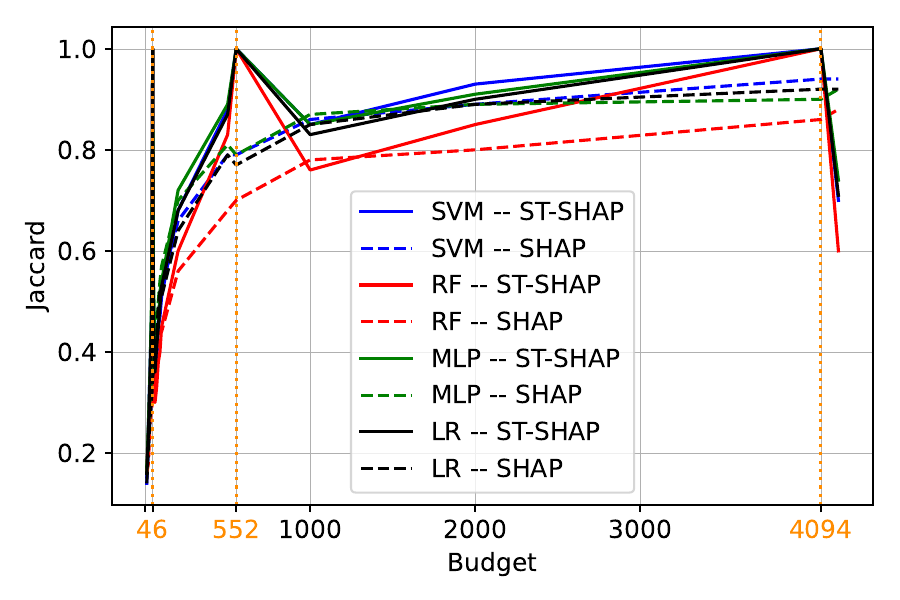}{(a) Jaccard} \,
    \Image[width=.406\linewidth]{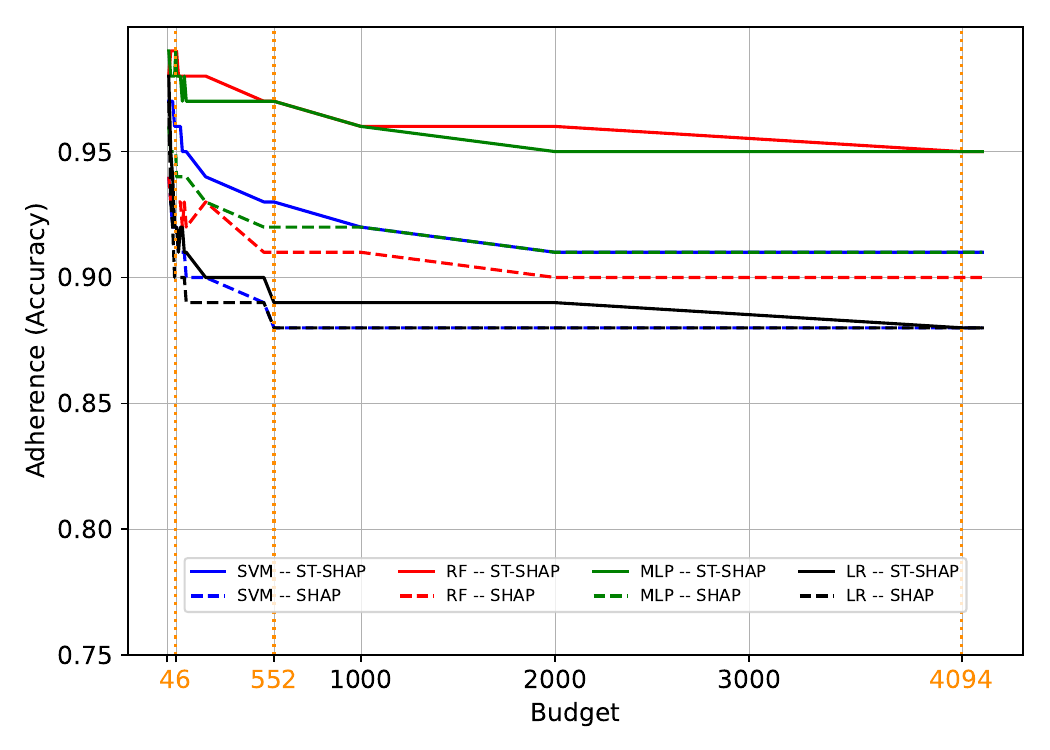}{(b) Adherence with the black-box} \,
    \captionof{figure}{SHAP and ST-SHAP on the HELOC dataset.}
    \vspace{.2cm}
    \label{fig:results_heloc}
\end{center}

  \begin{center}
    \Image[width=0.45\linewidth]{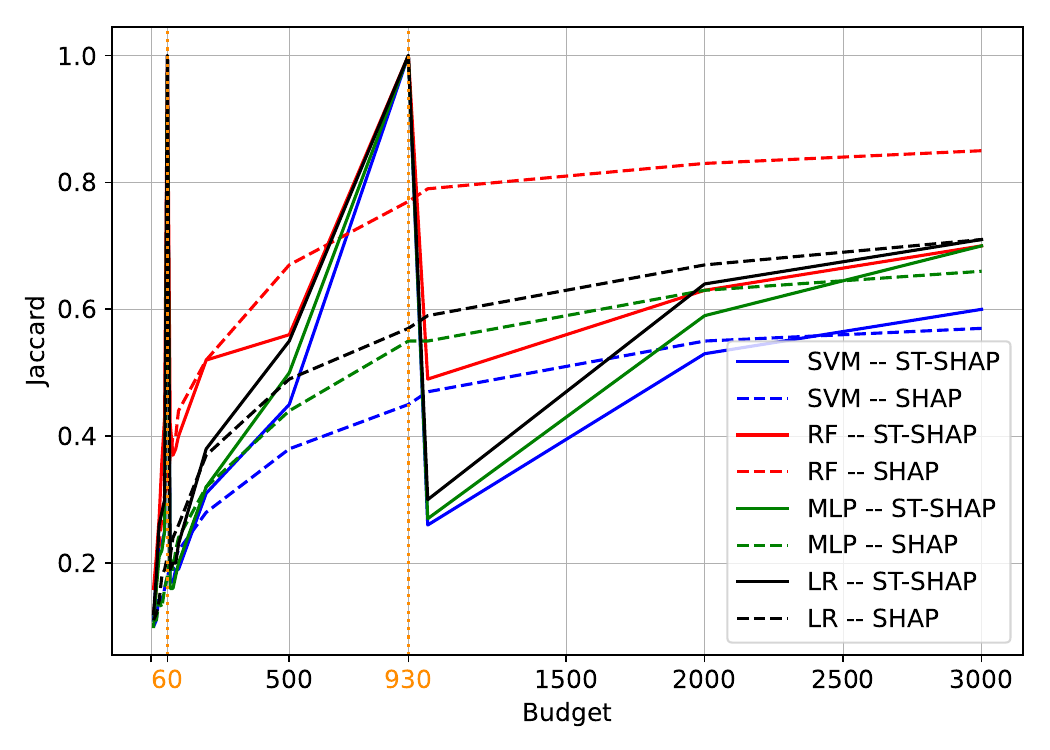}{(a) Jaccard} \,
    \Image[width=.406\linewidth]{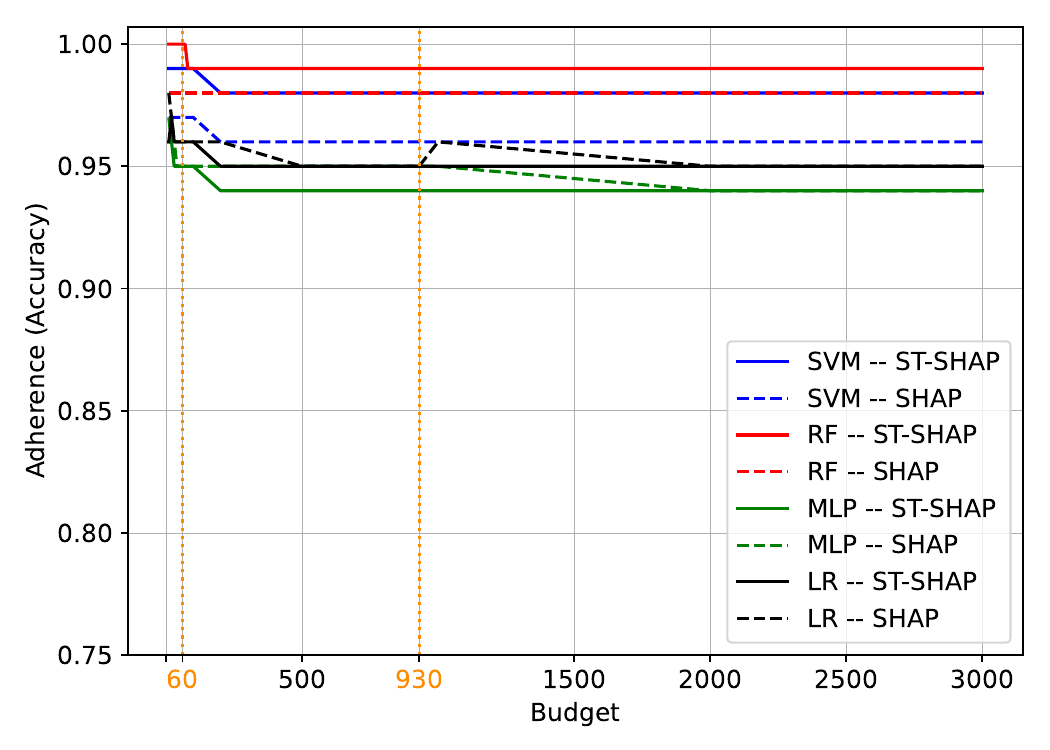}{(b) Adherence with the black-box} \,
    \captionof{figure}{SHAP and ST-SHAP on the Breast Cancer Wisconsin (Diagnostic) dataset.}
    \vspace{.2cm}
    \label{fig:results_wdbc}
\end{center}

\subsection{Comparison of execution times}
This section compares execution times for Layer 1, including the calculation of exact SHAP values, and for SHAP using a ``normal'' budget (approximating the recommended budget for SHAP).

\begin{table}[!htb]
\centering
\resizebox{0.8\linewidth}{!}{
    \begin{tabular}{lccccccccccc}
\toprule
         & \multicolumn{3}{c}{ST-SHAP Layer 1} &  & \multicolumn{3}{c}{Exact SHAP value} &  & \multicolumn{3}{c}{Kernel SHAP \footnotesize(budget = $2000$)} \\ \midrule
         & SVM        & RF         & MLP       &  & SVM        & RF        & MLP       &  & SVM         & RF          & MLP        \\ \cmidrule(lr){2-4} \cmidrule(lr){6-8} \cmidrule(lr){10-12}
Boston   & 0.003      & 0.032      & 0.007     &  & 0.22       & 0.154     & 0.17      &  & 0.083       & 0.087       & 0.06       \\
Dry bean & 0.031      & 0.02       & 0.019     &  & 7.382      & 0.78      & 1.07      &  & 0.28        & 0.098       & 0.103      \\
Adult    & 0.014      & 0.01       & 0.008     &  & 1.663      & 0.099     & 0.103     &  & 0.64        & 0.067       & 0.067      \\
Movie    & 0.01       & 0.028      & 0.008     &  & 6.802      & 4.88      & 4.62      &  & 0.083       & 0.088       & 0.062      \\ \bottomrule
\end{tabular}
}
\caption{Average computation time (in seconds) for explanations using ST-SHAP Layer 1, Kernel SHAP with a budget of $2000$ coalitions, and the exact SHAP value computed with all coalitions. The calculations are carried out using the entire test set for each dataset.}
\label{tab:exec_time_layer1_shapley}
\end{table}

Table \ref{tab:exec_time_layer1_shapley} presents the results obtained for the execution time. It is noticeable that the runtime of ST-SHAP Layer 1 is significantly more efficient than that required to compute the exact SHAP values. It is up to three orders of magnitude faster than computing the exact SHAP values.
Considering that the standard Kernel SHAP is recommended with a budget of $2000$ ($2\cdot M + 2^{11}$), utilizing ST-SHAP Layer 1 shows to be up to one order of magnitude faster than Kernel SHAP. This is so because the budget is significantly lower than $2000$. Moreover, as we saw previously, explanations calculated using only the coefficients from Layer 1 are stable, exhibit good adherence to the black-box model, and remain a good approximation of the exact SHAP value.

\end{document}